\documentclass{article}

\usepackage{microtype}
\usepackage{graphicx}
\usepackage{subfigure}
\usepackage{booktabs} 

\usepackage{hyperref}

\usepackage[accepted]{icml2023}

\usepackage{amsmath}
\usepackage{amssymb}
\usepackage{mathtools}
\usepackage{amsthm}

\usepackage[capitalize,noabbrev]{cleveref}

\theoremstyle{plain}
\newtheorem{theorem}{Theorem}[section]

\newtheorem{lemma}[theorem]{Lemma}
\newtheorem{corollary}[theorem]{Corollary}
\theoremstyle{definition}

\newtheorem{assumption}[theorem]{Assumption}
\theoremstyle{remark}

\usepackage[textsize=tiny]{todonotes}

\icmltitlerunning{Defects of Convolutional Decoder Networks in Frequency Representation}

\begin{document}

\twocolumn[
\icmltitle{Defects of Convolutional Decoder Networks in Frequency Representation}



\icmlsetsymbol{equal}{*}

\begin{icmlauthorlist}
\icmlauthor{Ling Tang}{equal,yyy}
\icmlauthor{Wen Shen}{equal,yyy}
\icmlauthor{Zhanpeng Zhou}{yyy}
\icmlauthor{Yuefeng Chen}{comp}
\icmlauthor{Quanshi Zhang}{yyy,cor}
\end{icmlauthorlist}

\icmlaffiliation{yyy}{Shanghai Jiao Tong University.}
\icmlaffiliation{comp}{Alibaba Group.}
\icmlaffiliation{cor}{Quanshi Zhang is the corresponding author. He is with the Department of Computer Science and Engineering, the John Hopcroft Center, at the Shanghai Jiao Tong University, China}

\icmlcorrespondingauthor{Quanshi Zhang}{zqs1022@sjtu.edu.cn}

\icmlkeywords{Machine Learning, ICML}

\vskip 0.3in
]



\printAffiliationsAndNotice{\icmlEqualContribution} 

\begin{abstract}
In this paper, we prove the representation defects of a cascaded convolutional decoder\footnote{Here, the decoder represents a typical network, whose feature map size is non-decreasing during the forward propagation.} network, considering the capacity of representing different frequency components of an input sample. 
We conduct the discrete Fourier transform on each channel of the feature map in an intermediate layer of the decoder network. Then, we extend the 2D circular convolution theorem to represent the forward and backward propagations through convolutional layers in the frequency domain. Based on this, we prove three defects in representing feature spectrums.
First, we prove that the convolution operation, the zero-padding operation, and a set of other settings all make a convolutional decoder network more likely to weaken high-frequency components. Second, we prove that the upsampling operation generates a feature spectrum, in which strong signals repetitively appear at certain frequencies. Third, we prove that if the frequency components in the input sample and frequency components in the target output for regression have a small shift, then the decoder usually cannot be effectively learned.
\end{abstract}

\section{Introduction}

In this study, we investigate the representation defect of a cascaded convolutional decoder\footnotemark[1] in generating features at different frequencies. That is, when we apply the discrete Fourier transform (DFT) to each channel of the feature map or the input sample, we try to prove which frequency components of each input channel are usually strengthened/weakened by the network.

\begin{figure*}
	\centering
	\includegraphics[width=\linewidth]{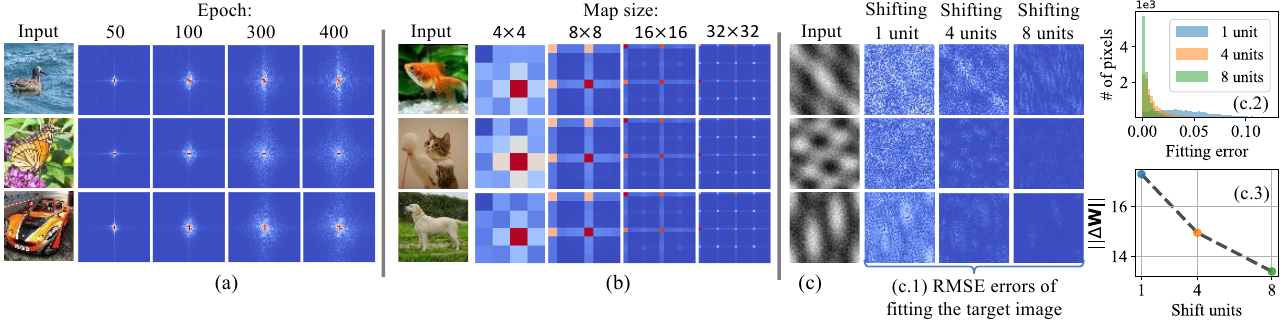}
	\vskip -0.1in
	\caption{Three representation problems with a cascaded convolutional decoder network. (a) The convolution operation and the zero-padding operation make the decoder usually learn low-frequency components first and then gradually learn higher frequencies. (b) For cascaded upconvolutional layers, the upsampling operation in the decoder repeats strong frequency components of the input to generate spectrums of upper layers. High-frequency components in magnitude maps in (b) are also weakened by the convolution operation after upsampling. We visualize the magnitude map of the feature spectrum, which is averaged over all channels. For clarity, we move low frequencies to the center of the spectrum map, and move high frequencies to the corners of the spectrum map. (c) The auto-encoder usually cannot be trained to fit the target output, whose specific frequency components have a small shift from the spectrum of the input image. We visualize the heatmap of fitting errors (c.1), \emph{i.e.}, the root mean squared error (RMSE), visualize the histogram of fitting errors over different pixels (c.2), and report the learning difficulty {\small$\lVert \Delta \textbf{W}\rVert$} (c.3). Here, results in (c.2) and (c.3) are averaged over different DNNs. Note that for magnitude maps in (a), we set the magnitude of the fundamental frequency to be the same with the magnitude of the second significant frequency.}
	\label{fig:bottleneck}
\end{figure*}

We extend the 2D circular convolution theorem to reformulate the forward propagation through multiple convolutional layers in the frequency domain. We find that both the forward propagation and the backward propagation in a convolutional network can be represented as the matrix multiplication on spectrums of the feature. Specifically, we mainly analyze a convolutional decoder, which only contains convolution operations without changing the size of feature maps. Then, based on the propagation in the frequency domain, we prove the following conclusions.

$\bullet$ \emph{Problem in representing high-frequency components.} We prove that both the convolution operation and the zero-padding operation make a cascaded convolutional decoder network more likely to weaken the high-frequency components of the input sample, if the convolution operation with a padding operation does not change the size of the feature map in a channel, as shown in Figure~\ref{fig:bottleneck}(a).
Besides, we also prove that the following three conditions further strengthen the above representation problem, including (1) a deep network architecture; (2) a small convolutional kernel size; and (3) a large absolute value of the mean value of convolutional weights.

$\bullet$ \emph{Problem in mistakenly repeating certain frequencies.} We find that the upsampling operation makes a cascaded convolutional decoder network generate a feature spectrum, in which strong signals repetitively appear at certain frequencies, as shown in Figure~\ref{fig:bottleneck}(b).

$\bullet$ \emph{Problem in fitting specific frequency components.} More crucially, we discover and prove that it is usually difficult to train an auto-encoder to fit the target image, if salient frequency components of the target output and those of the input have a small shift in the spectrum. Considering the continuous success of the auto-encoder in recent years, such a phenomenon is quite contrary to intuition.
As Figure~\ref{fig:bottleneck}(c) shows, a smaller shift between the input spectrum and the output spectrum usually leads to a higher difficulty in training the auto-encoder.

The above three problems just explain general trends towards generic problems of neural networks with convolution, zero-padding, and upsampling operations, instead of deriving a deterministic property of a specific network.

Although most conclusions are derived by ignoring ReLU operations in the decoder, we have conducted experiments, which have successfully verified such defects in different multi-layer decoder networks with ReLU layers. This proves the trustworthiness of our theorems. Note that we have not derived the property of max-pooling operations, so in this paper, it is difficult to extend such findings to neural networks for image classification.

\textbf{Discussions on two types of frequencies.}
People usually analyze feature representations of a network considering two types of frequencies. \cite{xu2019frequency,rahaman2019spectral} took the landscape of the loss function on all input samples as the time domain to analyze the frequency in the sample space. In comparison, we focus on the second type of frequency, \emph{i.e.}, we apply DFT to each channel of the intermediate-layer feature of a convolutional decoder and analyze defects in representing specific frequencies.

\section{Related work}

Although few previous studies directly prove a DNN's defects from the perspective of representing specific feature components, we still make a survey on research on the representation capacity of a DNN.

Some studies focused on a specific frequency that took the landscape of the loss function on all input samples as the time domain \citep{xu2019training,rahaman2019spectral,xu2019frequency,luo2019theory}. Based on such a specific frequency, they observed and proved a phenomenon namely Frequency Principle (F-Principle) that a DNN first quickly learned low-frequency components, and then relatively slowly learned the high-frequency ones. Ma \emph{et al.} \yrcite{ma2020machine} further explored the boundary of the F-Principle, beyond which the F-Principle did not hold anymore. Besides, Lin \emph{et al.} \yrcite{lin2019bandlimiting} empirically proposed to smooth out high-frequency components to improve the adversarial robustness. \textbf{In comparison, we focus on a fully different type of frequency}, \emph{i.e.}, the frequency \emph{w.r.t.} the DFT on an input image or a feature map.

In this direction, previous studies mainly experimentally analyzed the relationship between the learning of different frequencies and the robustness of a DNN.
Yin \emph{et al.} \yrcite{yin2019fourier} conducted a lot of experiments to analyze the robustness of a DNN \emph{w.r.t.} different frequencies of the image. They discovered that both adversarial training and Gaussian data augmentation improved the DNN's robustness to higher frequencies. Wang \emph{et al.} \yrcite{wang2020high} empirically proposed to remove high-frequency components of convolutional weights to improve the adversarial robustness. In comparison, we theoretically prove representation defects of DNNs in the frequency domain.

In fact, many studies explained the representation capacity of a DNN in the \textbf{time domain}. The information bottleneck hypothesis \citep{tishby2015deep,shwartz2017opening,wolchover2017new,amjad2019learning} showed that the learning process of DNNs was to retain the task-relevant input information and discarded the task-irrelevant input information. The lottery ticket hypothesis \citep{frankle2018lottery} showed that some initial parameters of DNNs inherently contributed more to the network output. The double-descent phenomenon \citep{2019Deep,2020Early} described the specific training process of DNNs that the loss first declined, then rose, and then declined again. DNNs with batch normalization were sometimes conflicted with the weight decay \citep{van2017l2, li2020understanding}.
DNNs were difficult to encode interactions between an intermediate number of input variables \citep{deng2021discovering}.

\section{Propagation in the frequency domain}\label{sec:dynamics}

\textbf{Preliminary 1, convolution operation.} Given a convolutional layer, let {\small$\textbf{W}^{[\textit{ker}=1]}$}, {\small$\textbf{W}^{[\textit{ker}=2]},\ldots$, $\textbf{W}^{[\textit{ker}=D]}$} denote {\small$D$} convolutional kernels of this layer, and let {\small$b^{[\textit{ker}=1]},b^{[\textit{ker}=2]},\ldots,b^{[\textit{ker}=D]}\in\mathbb{R}$} denote {\small$D$} bias terms.
Each {\small$d$}-th kernel {\small$\textbf{W}^{[\textit{ker}=d]}\in\mathbb{R}^{C\times  K\times K  }$} is of the kernel size {\small$K\times K$}, and {\small$C$} denotes the channel number. Accordingly, we apply these kernels on a feature {\small$\textbf{F}\in\mathbb{R}^{C \times M\times N }$} with {\small$C$} channels, and obtain the output feature {\small$\widetilde{\textbf{F}}\in\mathbb{R}^{D\times M'\times N'}$}, as follows.
\begin{small}
	\begin{equation}\label{eq:layerwise_conv}
		\begin{aligned}
			\widetilde{\textbf{F}} \!=\! \textit{Conv}(\textbf{F}), \ \ \textit{s.t.} \ \forall d,\ \widetilde{\textbf{F}}^{(d)} \!=\! \textbf{W}^{[\textit{ker}=d]}\otimes \textbf{F} + b^{[\textit{ker}=d]}\textbf{1}_{M'\times N'},
		\end{aligned}
	\end{equation}
\end{small}
where {\small$\widetilde{\textbf{F}}^{(d)}\in\mathbb{R}^{M'\times N'}$} denotes the feature map of the {\small$d$}-th channel. {\small$\otimes$} denotes the convolution operation. {\small$\textbf{1}_{M'\times N'}\in\mathbb{R}^{M'\times N'}$} is an all-ones matrix.

\textbf{Preliminary 2, discrete Fourier transform.} Given the $c$-th channel of the feature {\small$\textbf{F}\in\mathbb{R}^{C\times M\times N}$}, \emph{i.e.}, {\small$F^{(c)}\in\mathbb{R}^{ M\times N}$}, we use the discrete Fourier transform (DFT) \citep{sundararajan2001discrete} to compute the frequency spectrum of this channel, which is termed {\small$G^{(c)}\in\mathbb{C}^{M\times N}$}, as follows. {\small$\mathbb{C}$} denotes the algebra of complex numbers. 
\begin{small}
	\begin{equation}
		\begin{aligned}
			\forall u,v, \quad G_{uv}^{(c)}=\sum\nolimits_{m=0}^{M-1}\sum\nolimits_{n=0}^{N-1} F_{mn}^{(c)} e^{-i(\frac{um}{M}+\frac{vn}{N})2\pi}.
		\end{aligned}
	\end{equation}
\end{small}
Each frequency component at the frequency {\small$[u,v]$} is represented as a complex number, \emph{i.e.}, {\small$ G_{uv}^{(c)}\in\mathbb{C}$}. Let {\small$\textbf{G}=[G^{(1)},\ldots,G^{(C)}]\in\mathbb{C}^{C\times M\times N}$} denote the tensor of frequency spectrums of the {\small$C$} channels of {\small$\textbf{F}$}. We take the {\small$C$}-dimensional vector at the frequency {\small$[u, v]$} of the tensor {\small$\textbf{G}$}, \emph{i.e.}, {\small$\textbf{g}^{(uv)}=[G^{(1)}_{uv},G^{(2)}_{uv},\ldots,G^{(C)}_{uv}]^{\top}\in\mathbb{C}^C$}, to represent the frequency component {\small$[u, v]$} of the feature {\small$\textbf{F}$}. Frequency components closed to {\small$[0,0],[0,N-1],[M-1,0]$}, or {\small$[M-1,N-1]$} represent low-frequency signals, whereas frequency components closed to {\small$[\frac{M}{2},\frac{N}{2}]$} represent high-frequency signals.

\subsection{Propagation in frequency}

In this section, we extend the 2D circular convolution theorem \cite{jain1989fundamentals} to represent the forward propagation and the back-propagation in a cascaded convolutional network.

\begin{assumption}\label{ass1}
	\emph{Let us follow the setting in the 2D circular convolution theorem \cite{jain1989fundamentals}, which adds the circular padding operation assumption \cite{jain1989fundamentals} to the convolution operation. The convolution operation is conducted with a circular padding and with a stride size of 1, so as to avoid the convolution changing the size of the feature map. The circular padding is used to extend the last row and the last column of the feature map in each channel.}
\end{assumption}

\begin{theorem}\label{th2_layerwise}
	(Proof in Appendix~\ref{app:sec:th2}).
	According to Assumption~\ref{ass1}, the output feature {\small$\widetilde{\textbf{F}}\in\mathbb{R}^{D\times M \times N}$} has the same size as the input feature. Let {\small$\textbf{H}=[H^{(1)},H^{(2)},\ldots,H^{(D)}]\in\mathbb{C}^{D\times M\times N}$} denote a tensor consisting of {\small$D$} spectrums corresponding to the {\small$D$} channels of {\small$\widetilde{\textbf{F}}$}. Then, {\small$\textbf{H}$} can be computed as follows.
	\begin{small}
		\begin{equation}\label{eq:layerwise_spectrum_2}
			\begin{aligned}
				\textbf{h}^{(uv)}\!=\! 
				T^{(uv)}\textbf{g}^{(uv)}\!+\! \delta_{uv}MN\textbf{b} \ \ \ \
				\text{s.t.}\  \delta_{uv}\!=\!\begin{cases}
					1, \!\!& u \!=\! v \!=\! 0 \\
					0, \!\!  & \text{otherwise}
				\end{cases}
			\end{aligned}
		\end{equation}
	\end{small}
	where {\small$\textbf{h}^{(uv)} =[H^{(1)}_{uv},H^{(2)}_{uv},\ldots,H^{(D)}_{uv}]^{\top}\in\mathbb{C}^D$} denotes a column at the frequency {\small$[u,v]$} in the tensor {\small$\textbf{H}$}; {\small$T^{(uv)}\in\mathbb{C}^{D \times C}$} is a matrix of complex numbers and is exclusively determined by convolutional kernels {\small$\textbf{W}^{[\textit{ker}=1]},$ $\textbf{W}^{[\textit{ker}=2]},\ldots$, $\textbf{W}^{[\textit{ker}=D]}$}, {\small$T^{(uv)}_{dc}=\sum\nolimits_{t=0}^{K-1}\sum\nolimits_{s=0}^{K-1}W_{cts}^{[\textit{ker}=d]}e^{i(\frac{ut}{M}+\frac{vs}{N})2\pi}$}; {\small$\textbf{b}=[b^{(1)},b^{(2)},\ldots,b^{(D)}]^{\top}\in\mathbb{R}^D$} denotes the vector of bias terms.
\end{theorem}

To simplify the further proof, we temporarily investigate the spectrum propagation of a network with $L$ cascaded convolutional layers, but does not contain activation functions. Let us first discuss the trustworthiness of such a simplification. We have conducted experiments to show that all our findings in all theorems can also well explain the properties of an ordinary cascaded convolutional network with ReLU layers. As shown in Figure~\ref{fig:corollary}, for a network with ReLU layers, although the value derived from our theory was not exactly the same as the real value, experimental results still verified the conclusions of our theory. More crucially, experiments in Figures~\ref{fig:bottleneck}, \ref{fig:corollary}, \ref{fig:remark2_3}, and \ref{fig:icml_exp} were all conducted on ReLU networks.

\begin{figure}
	\centering
	\includegraphics[width=\linewidth]{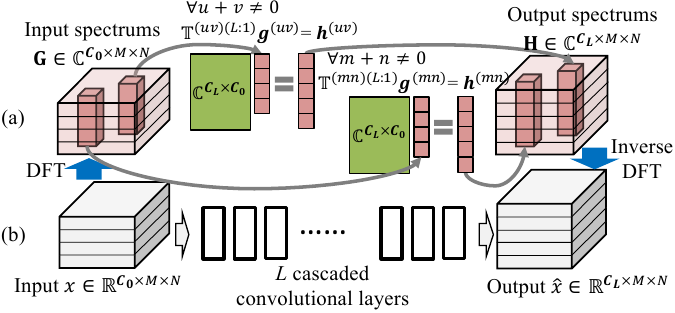}
	\vskip -0.05in
	\caption{Forward propagation in the frequency domain (a) and forward propagation in the time domain (b). The cascaded convolution operations on input {\small$x$} are essentially equivalent to matrix multiplication on spectrums {\small$\textbf{G}$} of the input.}
	\label{fig:convfrequency}
\end{figure}

\begin{figure*}[tbp]
	\centering
	\includegraphics[width=\linewidth]{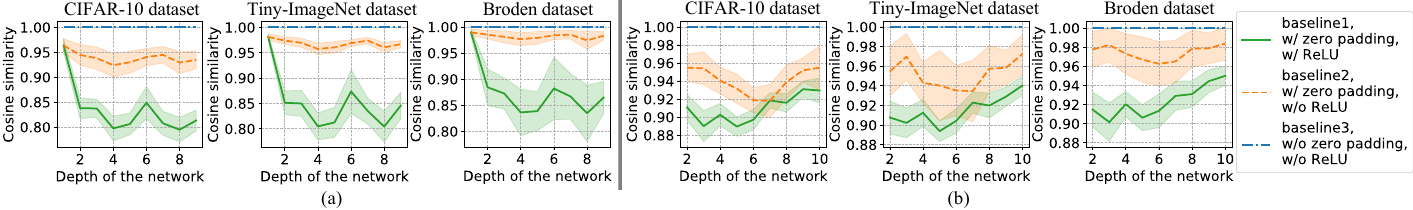}
	\vskip -0.05in
	\caption{(a) Fitness between the derived feature spectrums {\small$\textbf{H}$} in Corollary~\ref{co1_entire} and the real feature spectrums {\small$\textbf{H}^*$} measured in a real DNN. (b) Fitness between the derived change of {\small$T^{(l,uv)}$} in Corollary~\ref{co2_back} and the real {\small$T^{(l,uv)}$} measured in a real DNN. The shaded area represents the standard deviation.}
	\label{fig:corollary}
\end{figure*}

Let a convolutional network contain {\small$L$} cascaded convolutional layers. Each $l$-th layer contains {\small$C_l$} convolutional kernels, {\small$\textbf{W}^{(l)[\textit{ker}=1]},\textbf{W}^{(l)[\textit{ker}=2]},\ldots,\textbf{W}^{(l)[\textit{ker}=C_l]}\in\mathbb{R}^{C_{l-1} \times K\times K}$}, with {\small$C_l$} bias terms {\small$b^{(l,1)},b^{(l,2)},\ldots,b^{(l,C_l)}\in\mathbb{R}$}.
Let {\small$x\in\mathbb{R}^{C_0\times M\times N}$} denote the input sample. The network generates the output sample {\small$\widehat{x}=\text{net}(x)\in\mathbb{R}^{C_L\times M\times N}$}.
Then, we derive the forward propagation of spectrums of {\small$x$} to spectrums of {\small$\widehat{x}$} in the frequency domain, as follows.
\vskip 0.1in
\begin{corollary}\label{co1_entire}
	(Proof in Appendix~\ref{app:sec:co1})
	Let {\small$\textbf{G}=[ G^{(1)},G^{(2)},\ldots,G^{(C_0)}]\in\mathbb{C}^{C_0\times M\times N}$} denote frequency spectrums of the {\small$C_0$} channels of the input {\small$x$}. Then, based on Assumption~\ref{ass1}, spectrums of the image {\small$\widehat{x}$} generated by {\small$L$} cascaded convolutional layers, i.e., {\small$\textbf{H}=[H^{(1)},H^{(2)},\ldots,H^{(C_L)}]\in\mathbb{C}^{C_L\times M\times N}$}, are given as
	\begin{small}
		\begin{equation}\label{eq:forward}
			\textbf{h}^{(uv)}= 
			\boldsymbol{\mathbb{T}}^{(uv)(L:1)} \textbf{g}^{(uv)}+ \delta_{uv}\boldsymbol{\beta}
		\end{equation} 
	\end{small}
	where {\small$\textbf{g}^{(uv)}=[G_{uv}^{(1)},G_{uv}^{(2)},\ldots,G_{uv}^{(C_0)}]^{\top}\in\mathbb{C}^{C_0}$} and {\small $\textbf{h}^{(uv)}=[H_{uv}^{(1)},H_{uv}^{(2)},\ldots,H_{uv}^{(C_L)}]^{\top}\in\mathbb{C}^{C_L}$} denote vectors at the frequency {\small$[u,v]$} in tensors {\small$\textbf{G}$} and {\small$\textbf{H}$}, respectively. {\small$\boldsymbol{\mathbb{T}}^{(uv)(L:1)}=T^{(L,uv)}\cdots T^{(2,uv)}T^{(1,uv)}\in\mathbb{C}^{C_L \times C_0}$}.
	{\small$\boldsymbol{\beta}=MN\big(\textbf{b}^{(L)}+\sum_{j=2}^{L}\boldsymbol{\mathbb{T}}^{(00)(L:j)}\textbf{b}^{(j-1)} \big)\in\mathbb{C}^{C_L}$. $\textbf{b}^{(l)}=[b^{(l,1)},b^{(l,2)},\ldots,b^{(l,C_l)}]^{\top}\in\mathbb{R}^{C_l}$} denotes the vector of bias terms of {\small$C_l$} convolutional kernels in the $l$-th layer.
\end{corollary}

\emph{Understanding the cascaded convolution operations in the frequency domain.} Corollary~\ref{co1_entire} means that conducting multiple cascaded convolution operations on an input {\small$x$} is essentially equivalent to conducting matrix multiplication on spectrums of {\small$x$}. As Figure~\ref{fig:convfrequency} shows, for all frequencies except for the fundamental frequency, we have the output spectrum {\small$\textbf{h}^{(uv)}= 
	\boldsymbol{\mathbb{T}}^{(uv)(L:1)} \textbf{g}^{(uv)}$}.

Besides, the learning of parameters {\small$\textbf{W}^{(l)}$} affects the matrix {\small$T^{(l,uv)}$}. Therefore, we further reformulate the change of {\small$T^{(l,uv)}$} during the learning process, as follows.

\vskip 0.1in
\begin{corollary}\label{co2_back}
	(Proof in Appendix~\ref{app:sec:co2})
	Based on Assumption~\ref{ass1}, the change of each frequency component {\small$T^{(l,uv)}$} during the learning process is reformulated, as follows.
	\begin{small}
		\begin{equation}\label{eq:back1}
			\begin{aligned}
				(\Delta T^{(l,uv)})^{\top} &= - \eta MN 
				\sum\limits_{u'=0}^{M-1}\sum\limits_{v'=0}^{N-1}\chi_{u'v'uv} \Big( \overline{\boldsymbol{\mathbb{T}}}^{(u'v')(l-1:1)}   \overline{\textbf{g}}^{(u'v')}  \\ 
				&  \!\!\!\!\!\! + \delta_{u'v'}\overline{\boldsymbol{\beta}}' \Big) \frac{\partial \textit{Loss}}{\partial (\overline{\textbf{h}}^{(u'v')})^{\top}}  \overline{\boldsymbol{\mathbb{T}}}^{(u'v')(L:l+1)};
			\end{aligned}
		\end{equation}
	\end{small}
	where
	\small$\chi_{u'v'uv}=\frac{1}{MN}\frac{\sin(\frac{K(u-u')\pi}{M})}{\sin(\frac{(u-u')\pi}{M})} \frac{\sin(\frac{K(v-v')\pi}{N})}{\sin(\frac{(v-v')\pi}{N})} \cdot e^{i(\frac{(K-1)(u-u')}{M}+\frac{(K-1)(v-v')}{N})\pi} \in\mathbb{C}$ is a coefficient; {\small$\boldsymbol{\mathbb{T}}^{(u'v')(l-1:1)}=T^{(l-1,u'v')}\cdots T^{(2,u'v')} T^{(1,u'v')} \in\mathbb{C}^{C_{l-1} \times C_0}$; $\boldsymbol{\mathbb{T}}^{(u'v')(L:l+1)}=T^{(L,u'v')} \cdots T^{(l+1,u'v')} \in\mathbb{C}^{C_L \times C_l}$;
		$\boldsymbol{\beta}'=MN\big(  \textbf{b}^{(l-1)} +\sum_{j=2}^{l-1}\boldsymbol{\mathbb{T}}^{(00)(l-1:j)}\textbf{b}^{(j-1)} \big) \in\mathbb{C}^{C_{l-1}}$}; {\small$\overline{\textbf{g}}^{(u'v')}$} denotes the conjugate of {\small$\textbf{g}^{(u'v')}$}; {\small$\eta$} is the learning rate.		
\end{corollary}

\subsection{Experimental verification of Corollaries~\ref{co1_entire} and~\ref{co2_back}}

To verify the correctness of Corollary~\ref{co1_entire}, we computed the similarity between real spectrums {\small$\textbf{H}^*=[H^{*(1)},H^{*(2)},\cdots]$} measured by applying the DFT to the real network output, and spectrums {\small$\textbf{H}=[H^{(1)},H^{(2)},\cdots]$} derived in Corollary~\ref{co1_entire}. Specifically, we measured the cosine similarity {\small $\textit{sim}(\textbf{H}^*,\textbf{H}) = \mathbb{E}_c [\textbf{cos}(\textbf{vec}(\textit{mag}( H^{*(c)})),\textbf{vec}( \textit{mag}(H^{(c)} )))]$}, where {\small$\textbf{vec}(\cdot)$} represented the vectorization of a matrix, and {\small$\textit{mag}(\cdot)$} transferred a complex-valued matrix to a real-valued magnitude matrix\footnote{The function {\small$B=\textit{mag}(A)$} returns a matrix, where each element {\small$B_{ij}\in\mathbb{R}$} represents the magnitude of {\small$A_{ij}\in \mathbb{C}$}.}.

To this end, we constructed the following three baseline networks to verify whether Corollary~\ref{co1_entire} derived from specific assumptions could also objectively reflect real forward propagations in real neural networks. The first baseline network contained 10 convolutional layers. Each convolutional layer applied zero-paddings and was followed by an ReLU layer. Each convolutional layer contained 16 convolutional kernels (kernel size was {\small$3\times 3$}) with 16 bias terms. We set the stride size of the convolution operation to 1. The second baseline network was constructed by removing all ReLU layers from the first baseline network, which was closer to the assumption in Corollary~\ref{co1_entire}. The third baseline network was revised from the second baseline network by replacing all zero-paddings with circular paddings. The third baseline network followed the exact assumption in Corollary~\ref{co1_entire}.

Figure~\ref{fig:corollary}(a) reports {\small$\textit{sim}(\textbf{H}^*,\textbf{H})$} that was measured on spectrums in different layers and averaged over all samples. The similarity between real spectrums and derived spectrums was large for all the three baseline networks, which verified Corollary~\ref{co1_entire}. Note that the cosine similarity was computed based on high-dimensional vectors with as many as {\small$32^2$}, {\small$64^2$} or {\small$224^2$} dimensions (determined by the dataset). For such high-dimensional vectors, a similarity greater than 0.8 was already significant enough to verify the practicality of our theory\footnote{Please see Appendix~\ref{app:sec:dimension} for the curse of dimension.}.

Besides, in order to verify Corollary~\ref{co2_back}, we also measured the similarity between the real change of {\small$T^{(l,uv)}$} computed by measuring real network parameters, termed {\small$\Delta^* T^{(l,uv)}$}, and the change of {\small$T^{(l,uv)}$} derived with assumptions in Corollary~\ref{co2_back}, termed {\small$\Delta T^{(l,uv)}$}. The similarity was also computed\footnotemark[2] as {\small $\textit{sim}(\Delta^* T^{(l,uv)},\Delta T^{(l,uv)}) = \mathbb{E}_c [\textbf{cos}(\textbf{vec}(\textit{mag}(\Delta^* T^{(l,uv)})),\textbf{vec}(\textit{mag}(\Delta T^{(l,uv)})))]$}. The verification was also conducted on the above three baseline networks. Figure~\ref{fig:corollary}(b) reports {\small$\forall l,\textit{sim}(\Delta^* T^{(l,uv)},\Delta T^{(l,uv)})$} averaged over all samples. The similarity was greater than 0.88 for all three baseline networks, which was large\footnotemark[3] enough to verify Corollary~\ref{co2_back}.

\section{Representation problems}\label{sec:bottleneck}

In this section, we aim to prove three defects in the frequency representation with a cascaded convolutional decoder network. Note that unlike previous studies \cite{xu2019frequency,rahaman2019spectral} extracting frequent components in the sample space, we focus on a more commonly-used frequency, \emph{i.e.}, applying DFT to each channel of the intermediate-layer feature.

\subsection{Effects of the convolution operation} 

Given an initialized, cascaded, convolutional decoder\footnotemark[1] network with $L$ convolutional layers, let us focus on the behavior of the decoder network in the early epochs of training. 
We notice that each element in the matrix {\small$T^{(l,uv)}$}, \emph{i.e.}, {\small$T_{dc}^{(l,uv)}$}, is exclusively determined by the $c$-th channel of the $d$-th kenel {\small$W^{(l)[\textit{ker}=d]}_{c,0:K-1,0:K-1}\in\mathbb{R}^{K\times K}$}, according to Theorem~\ref{th2_layerwise}. Because parameters in {\small$W^{(l)}$} in the decoder network are initialized to random noises, we can consider that all elements in {\small$T^{(l,uv)}$} are irrelevant to each other, \emph{i.e.}, {\small$\forall d\ne d',c\ne c',T_{dc}^{(l,uv)}$} is irrelevant to {\small$T_{d'c'}^{(l,uv)}$}. Similarly, since different layers' parameters {\small$W^{(l)}$} are irrelevant to each other in the initialized decoder, we can consider that elements in different layers' {\small$T^{(l,uv)}$} are irrelevant to each other, \emph{i.e.}, $\forall l\ne l'$, elements in {\small$T^{(l,uv)}$} and elements in {\small$T^{(l',uv)}$} are irrelevant to each other. Moreover, since the early training of a DNN mainly modifies a few parameters according to the lottery ticket hypothesis \cite{frankle2018lottery}, we can still assume such irrelevant relationships in early epochs, as follows.

\vskip 0.1in
\begin{assumption}\label{ass:T}
	\emph{(Proof in Appendix~\ref{app:sec:th4})
		We assume that all elements in {\small$T^{(l,uv)}$} are irrelevant to each other, and {\small$\forall l\ne l'$}, elements in {\small$T^{(l,uv)}$} and {\small$T^{(l',uv)}$} are irrelevant to each other in early epochs.
		\begin{small}
			\begin{equation}
				\begin{aligned}
					\forall d&\ne d'; \forall c\ne c',\\ 
					&\mathbb{E}_{\textbf{W}^{(l)}}[T_{dc}^{(l,uv)}T_{d'c'}^{(l,uv)}] = \mathbb{E}_{\textbf{W}^{(l)}}[T_{dc}^{(l,uv)}]\mathbb{E}_{\textbf{W}^{(l)}}[T_{d'c'}^{(l,uv)}]
				\end{aligned}
			\end{equation}
		\end{small}
		\begin{small}
			\begin{equation}
				\begin{aligned}
					\forall l,d,c,d',c',& \ \
					\mathbb{E}_{\textbf{W}^{(l)},\ldots,\textbf{W}^{(1)}}[T_{dc}^{(l,uv)}\mathbb{T}_{d'c'}^{(uv)(l-1:1)}] = \\ &\mathbb{E}_{\textbf{W}^{(l)}}[T_{dc}^{(l,uv)})]\mathbb{E}_{\textbf{W}^{(l-1)},\ldots,\textbf{W}^{(1)}}[\mathbb{ T}_{d'c'}^{(uv)(l-1:1)}]
				\end{aligned}
			\end{equation}
		\end{small}
		Besides, according to experimental experience, the mean value of all parameters in {\small$\textbf{W}^{(l)}$} usually has a small bias during the training process, instead of being exactly zero. Therefore, let us assume that in early epochs, each parameter in {\small$\textbf{W}^{(l)}$} is sampled from a Gaussian distribution {\small$N(\mu_l,\sigma_l^2)$}.}
\end{assumption}

Note that we also experimentally verify that Assumption~\ref{ass:T} can be also applied to fully trained DNNs, besides DNNs trained after early epochs. Please see Appendix~\ref{app:sec:ass} for details.

According to {\small$\textbf{h}^{(uv)}=\boldsymbol{\mathbb{T}}^{(uv)(L:1)} \textbf{g}^{(uv)}+\delta_{uv}MN\textbf{b}$} in Corollary~\ref{co1_entire},
\textbf{we investigate the magnitude of {\small$\boldsymbol{\mathbb{T}}^{(uv)(L:1)}$} as an indicator to measure the strength of the network encoding this specific frequency component {\small$\textbf{g}^{(uv)}$}.}

\begin{theorem}\label{th:TTT}
	(Proof in Appendix~\ref{app:sec:th4})
	Let us focus on the simplest case that each convolutional layer only contains a feature map with a single channel, i.e., {\small$\forall l, C_l=1$}. Based on Assumption~\ref{ass:T}, {\small$\boldsymbol{\mathbb{T}}^{(uv)(L:1)}\in\mathbb{C}$} is computed as {\small$T^{(L,uv)}\cdots T^{(2,uv)}T^{(1,uv)}$}, which is the product of {\small$L$} complex numbers. Because each complex number {\small$T^{(l,uv)}$} follows a Gaussian distribution\footnote{The Gaussian distribution of complex numbers has three parameters {\small$\mu\in\mathbb{C},\sigma^2\in\mathbb{R}$} and {\small$r\in\mathbb{C}$}, which control the mean value, the variance, and the diversity of the phase of the sampled complex number, respectively.}, the mean value of {\small$\boldsymbol{\mathbb{T}}^{(uv)(L:1)}$} is {\small$\prod\nolimits_{l=1}^L \mu_lR_{uv}\in\mathbb{C}$}, where {\small$R_{uv} = \frac{\sin(\frac{uK\pi}{M})}{\sin(\frac{u\pi}{M})} \frac{\sin(\frac{vK\pi}{N})}{\sin(\frac{v\pi}{N})} e^{i(\frac{(K-1)u}{M}+\frac{(K-1)v}{N})\pi}\in\mathbb{C}$} is a complex coefficient; {\small$0 \le \lvert R_{uv} \rvert \le K^2$}. The logarithm of the second-order moment is given as {\small$\log\textit{SOM}(\boldsymbol{\mathbb{T}}^{(uv)(L:1)})=\sum\nolimits_{l=1}^L \log(\lvert\mu_lR_{uv}\rvert^2+K^2\sigma_l^2)\in\mathbb{R}$}.
\end{theorem}

Theorem~\ref{th:TTT} tells us the following five conclusions.

(1) The magnitude of {\small$\boldsymbol{\mathbb{T}}^{(uv)(L:1)}$}, which is measured using the second-order moment {\small$\textit{SOM}(\boldsymbol{\mathbb{T}}^{(uv)(L:1)})$}, increases along with the following four terms, including the absolute value of the expectation {\small$\lvert\mu_l\rvert$}, the magnitude of the complex coefficient {\small$\lvert R_{uv} \rvert$}, the kernel size {\small$K$}, and the variance {\small$\sigma_l^2$}.

(2) For each frequency component {\small$[u,v]$}, the magnitude of {\small$\boldsymbol{\mathbb{T}}^{(uv)(L:1)}$} will exponentially increase along with the depth $L$ of the network.
{\bf We can consider that each layer' {\small$T^{(l,uv)}$} has independent effects {\small$\log(\lvert\mu_lR_{uv}\rvert^2+K^2\sigma_l^2)$} on} {\small$\log\textit{SOM}(\boldsymbol{\mathbb{T}}^{(uv)(L:1)})=\sum\nolimits_{l=1}^L \log(\lvert\mu_lR_{uv}\rvert^2+K^2\sigma_l^2)$}.
We admit that such an conclusion is derived from the second-order moment of {\small$\boldsymbol{\mathbb{T}}^{(uv)(L:1)}$}, instead of a deterministic claim for a specific neural network. Nevertheless, according to the Law of Large Numbers, {\small$\textit{SOM}(\boldsymbol{\mathbb{T}}^{(uv)(L:1)})$} is still a convincing metric to reflect the average significance of {\small$\boldsymbol{\mathbb{T}}^{(uv)(L:1)}$}.

For the general case that each convolutional kernel contains more than one channel, \emph{i.e.}, {\small$\forall l,C_l>1$}, the magnitude of {\small$\boldsymbol{\mathbb{T}}^{(uv)(L:1)}$} also approximately exponentially increases along with the network depth with a quite complicated analytic solution. Please see Appendix~\ref{app:sec:th4} for the proof.

(3) The convolution operation makes a cascaded convolutional decoder network more likely to weaken the high-frequency
components of the input sample, if the convolution operation
does not change the feature map size. Especially, when the decoder network is deep, such a problem is more significant.
See Appendix~\ref{app_discussion_L} for more discussions.

(4)	If the expectation {\small$\mu_l$} of convolutional weights in each $l$-th layer has a large absolute value {\small$\lvert\mu_l\rvert$}, then the decoder network is less likely to learn high-frequency components. Please see Appendix~\ref{app_discussion_mu} for more discussions.

(5) If the convolutional kernel size $K$ is small, then the decoder network is less likely to learn high-frequency components. Please see Appendix~\ref{app_discussion_K} for more discussions.

Experiments in Section~\ref{sec:exp:b1} have verified the above conclusions in the general case that each convolutional layer contains more than one feature map.

\subsection{Effects of the zero-padding operation} 

To simplify the proof, let us consider the following one-side zero-padding.
Given each $c$-th channel {\small$F^{(c)}\in\mathbb{R}^{M\times N}$} of the feature map, the zero-padding puts zero values at the edge of {\small$F^{(c)}$}, so as to obtain a new feature {\small$\widetilde{F}^{(c)}\in\mathbb{R}^{M'\times N'}$}.
\begin{small}
	\begin{equation}
		\forall m,n, \ \ \widetilde{F}^{(c)}_{mn}=\begin{cases}
			F^{(c)}_{mn}, & 0\le m<M,\ 0\le n<N \\
			0, & M\le m<M',\ N\le n<N'
		\end{cases}
	\end{equation}
\end{small}
We have proven that the zero-padding operation boosts magnitudes of low-frequency components of feature spectrums of the feature map, as shown in Theorem~\ref{th:zero_padding}.

\vskip 0.1in
\begin{theorem}\label{th:zero_padding}
	(Proof in Appendix~\ref{app:sec:th5})
	Let each element in each $c$-th channel {\small$F^{(c)}$} of the feature map follows the Gaussian distribution {\small$\mathcal{N}(a, \sigma^2)$}. {\small$G^{(c)}\in\mathbb{C}^{M\times N}$} denotes the frequency spectrum of {\small$F^{(c)}$}, and {\small$H^{(c)}\in\mathbb{C}^{M'\times N'}$} denotes the frequency spectrum of the output feature {\small$\widetilde{F}^{(c)}$} after applying zero-padding on {\small$F^{(c)}$}. Then, the zero-padding on {\small$F^{(c)}$} boosts the second-order moment (SOM) of each frequency component at {\small$[u,v]$} as follows, whose strength is measured by averaging over different sampled features.
	
	\begin{small}
		\begin{equation}\label{eq:zero_padding}
			\begin{aligned}
				&\forall 0\le u<M,\ 0\le v<N, u+v\ne 0 \\ 
				& \qquad\qquad \textit{SOM}(H_{uv}^{(c)}) - \textit{SOM}(G_{uv}^{(c)}) =  a^2  \tau_{uv}^2  , \\
			\end{aligned}	 
		\end{equation}
	\end{small}
	where {\small$\textit{SOM}(H_{uv}^{(c)})=\mathbb{E}[H_{uv}^{(c)}\overline{H}_{uv}^{(c)}]$}  denotes the second-order moment of {\small$H_{uv}^{(c)}$};
	{\small$\tau_{uv} = \frac{\sin (\frac{Mu\pi}{M'})}{\sin(\frac{u\pi}{M'})}  \frac{\sin(\frac{Nv\pi}{N'})}{\sin(\frac{v\pi}{N'})}\in\mathbb{R}$}. Note that for the fundamental frequency {\small$u=v=0$}, {\small$\textit{SOM}(H_{00}^{(c)}) = \textit{SOM}(G_{00}^{(c)})$}.
\end{theorem}

\textbf{(Conclusion)} According to the rule of the forward propagation in Equation~(\ref{eq:forward}) and the change of {\small$T^{(l,uv)}$} in Equation (\ref{eq:back1}), the zero-padding operation boosts the SOM of low-frequency components, because {\small$\tau_{uv}^2$} is large for low frequencies. This exhibits the trend of encoding low-frequency components of the input sample.

\subsection{Effects of the upsampling operation}

Let the $l$-th intermediate-layer feature map {\small$\textbf{F}\in\mathbb{R}^{C_l\times M_0 \times N_0}$} pass through an upsampling layer to extend its width and height to {\small$M \times N$}, subject to $M=M_0\cdot \textit{ratio}$, $N=N_0\cdot \textit{ratio}$ as follows.
\begin{small}
	\begin{equation}\label{eq:upconv}
		\begin{aligned}
			&\forall c, m^*, n^*, \\
			&\widetilde{F}_{m^*n^*}^{(c)}\!=\! \begin{cases}
				F_{mn}^{(c)},	& \textrm{mod}(m^*,\textit{ratio})\!=\! \textrm{mod}(n^*,\textit{ratio})\!=\!0\\
				0,	& \textit{otherwise} 
			\end{cases}\\ 
			&\textit{s.t.}\ \ m=\frac{m^*}{\textit{ratio}};\ n=\frac{n^*}{\textit{ratio}}
		\end{aligned}
	\end{equation}
\end{small}
\begin{theorem}\label{th_upconv}
	(Proof in Appendix~\ref{app:sec:th6})
	Let {\small$\textbf{G}=[G^{(1)},G^{(2)},\ldots,G^{(C_l)}]\in\mathbb{C}^{C_l\times M_0 \times N_0}$} denote spectrums of the {\small$C_l$} channels of feature {\small$\textbf{F}$}. Then, spectrums {\small$\textbf{H}=[H^{(1)},H^{(2)},\ldots,H^{(C_l)}]\in\mathbb{C}^{C_l\times M\times N}$} of the output feature {\small$\widetilde{\textbf{F}}$} can be computed as follows.
	\begin{small}
		\begin{equation}\label{eq:upsample_g}
			\begin{aligned}
				\forall c,u,v,\quad
				&H_{u+(s-1)M_0,v+(t-1)N_0}^{(c)}=G_{uv}^{(c)} \\ 
				&\text{s.t.}\ \ s=1,\ldots,\frac{M}{M_0};\ t=1,\ldots,\frac{N}{N_0}
			\end{aligned}	
		\end{equation}
	\end{small}
\end{theorem}

Theorem~\ref{th_upconv} shows that the upsampling operation repeats the strong magnitude of the fundamental frequency {\small$G_{00}^{(c)}$} of the lower layer to different frequency components {\small$\forall c,H_{u^*v^*}^{(c)}$} of the higher layer, where {\small$u^*=0, M_0,2 M_0, \ldots;v^* =0, N_0,2 N_0,\ldots$}. Besides Figure~\ref{fig:bottleneck}(b), Appendix~\ref{app:sec:exp:upsample} shows such a phenomenon on more datasets.

\textbf{(Conclusion)} The upsampling operation makes the upconvolution operation generate a feature spectrum, in which strong signals of the input periodically appear at certain frequencies. Such strong periodic signals hurt the representation capacity of the network.

More crucially, according to the spectrum propagation in Corollary~\ref{co1_entire}, such periodic frequency components can be further propagated to upper layers. Thus, Corollary~\ref{co1_entire} may provide some clues to differentiate real samples and the generated samples.

\begin{figure*}
	\centering
	\includegraphics[width=0.95\linewidth]{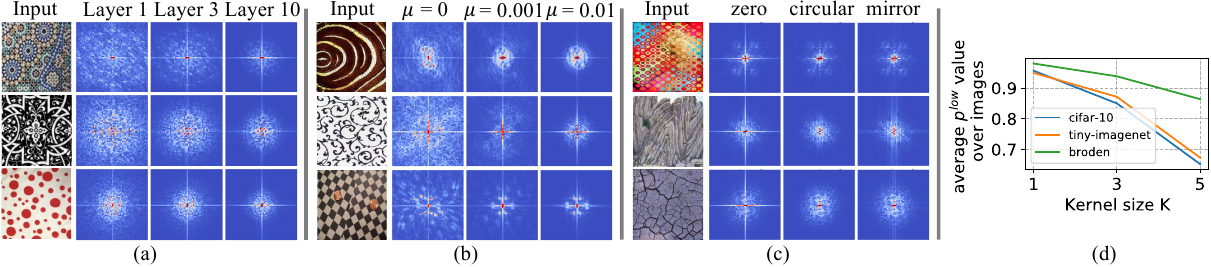}
	\vskip -0.05in
	\caption{(a) A higher layer of a network usually generated features with more low-frequency components, but with less high-frequency components. (b) A network whose convolutional weights have a mean value significantly biased from 0 usually strengthened low-frequency components, but weakened high-frequency components. (c) A network with zero-padding operations usually strengthened more low-frequency components than a network with circular padding operations. (d) A small kernel size {\small$K$} usually made the network learn a higher proportion {\small$p^{\text{low}}$} of low-frequency components. Here, each magnitude map of the feature spectrum was averaged over all channels. For clarity, we moved low frequencies to the center of the spectrum map, and moved high frequencies to corners of the spectrum map. Besides, we only visualized components in the center of the spectrum map with the range of relatively low frequencies\footnotemark[7] in {\small$\Omega^{\text{low}}$} for clarity.}
	\label{fig:remark2_3}
\end{figure*}

\subsection{Difficulty of representing specific frequencies}

Based on the propagation rule in frequency in Section~\ref{sec:dynamics}, we discover a further counter-intuitive phenomenon, \emph{i.e.}, in the scenario of an auto-encoder, if salient frequency components in the input sample and salient frequency components in the target output for regression have a small shift, then the decoder usually cannot be effectively learned.

Let us consider the input {\small$x\in\mathbb{R}^{M\times N}$} with a single channel to simplify the proof. 
Then, {\small$G\in \mathbb{C}^{M\times N}$}, {\small$H\in \mathbb{C}^{M\times N}$}, and {\small$H^*\in \mathbb{C}^{M\times N}$} denote spectrums of the input, the output, and the target image to fit, respectively. In a traditional auto-encoder, people usually set the target image the same as the input image, thereby {\small$H^*=G$}. 
Whereas, let us slightly shift a salient frequency component {\small$[u_1,v_1]$}, which is denoted by {\small$G_{u_1v_1}\in \mathbb{C}$}, to its neighboring frequency {\small$[u_2,v_2]$} to construct {\small$H^*$} and obtain the target image. 
Because the frequency component {\small$G_{u_1v_1}$} is salient, we can consider there is a significant increase of {\small$\lvert H^*_{u_2v_2}\rvert$} and a significance decrease of {\small$\lvert H^*_{u_1v_1}\rvert$}, compared with the traditional setting of {\small$H^*=G$}. Thus, we apply the following setting as a typical case, which simplifies the analysis of the representation problem. That is, {\small$H^*_{u_1v_1} = (1-A)G^{(u_1v_1)}$}, and {\small$H^*_{u_2v_2} = (1+\overline{A})G^{(u_2v_2)}$}, where {\small$A=\alpha e^{i\phi}$}, and {\small$\overline{A}$} denotes the conjugate of {\small$A$}; {\small$\alpha>0$}; {\small$\phi<\frac{\pi}{2}$}.
In this way, we can decrease the significance of {\small$H^*_{u_1v_1}$} and increase the significance of {\small$H^*_{u_2v_2}$}.

According to Corollary~\ref{co1_entire}, learning an identify function {\small$H=G$} is not difficult for an auto-encoder, because we can just make {\small$\forall u,v, \mathbb{T}^{(uv)(L:1)} = I$}. Thus, to investigate the extreme difficulty of learning specific frequencies, let us take the auto-encoder that models the identify function as the baseline network. Then, we further tune the auto-encoder to fit the image with a shifted spectrum {\small$H^*$} and compute the weight changes {\small$\Delta\textbf{W}$} as the cost of network training.
A large weight change\footnote{We represent parameters of multiple layers as a vector.} {\small$\lVert\Delta\textbf{W}\rVert$} indicates the high difficulty of fitting {\small$H^*$}. 
Note that we must ensure that parameters {\small$\textbf{W}$} are real-valued, instead of being complex-valued, when we train the auto-encoder. Theorem~\ref{th:cost} proves a case that {\small$\textbf{W}$} is optimized to satisfy {\small$H_{u_1v_1}=H^*_{u_1v_1}$} and {\small$H_{u_2v_2}=H^*_{u_2v_2}$} simultaneously.

\vskip 0.1in
\begin{theorem}\label{th:cost}
	(Proof in Appendix~\ref{app:sec:difficulty})
	Let us consider the objective function in the form {\small$\lambda_1\lvert H_{u_1v_1}-H^*_{u_1v_1} \rvert^2 + \lambda_2\lvert H_{u_2v_2}-H^*_{u_2v_2} \rvert^2$}. We prove specific constrains of {\small$\lambda_{1}$}, {\small$\lambda_{2}$}, and {\small$A$} that make the auto-encoder learnable (i.e., ensuring {\small$\Delta\textbf{W}$} is real-valued) and make the objective function can reach zero by a single step of gradient descent, which are shown in Equations (\ref{a1}) and (\ref{a2}) in Appendix~\ref{app:sec:difficulty}. Then, we prove the significance of the weight change {\small$\Delta\textbf{W}$}, as follows.
	\vskip -0.1in
	\begin{small}
		\begin{equation}
			\lVert \Delta\textbf{W} \rVert \propto  \frac{\alpha MN}{K^2- \frac{\sin(\frac{K(u_2-u_1)\pi}{M} )\sin(\frac{K(v_2-v_1)\pi }{N})}{\sin(\frac{(u_2-u_1)\pi}{M} )\sin(\frac{(v_2-v_1)\pi}{N})}} 
		\end{equation}
	\end{small}
\end{theorem}
We use the norm of {\small$\Delta\textbf{W}$} to measure the optimization cost (difficulty) to push the auto-encoder to fit a target image, one of whose frequency component is slightly shifted. Theorem~\ref{th:cost} shows that the 
learning difficulty is significantly boosted (\emph{i.e.}, {\small$\lVert \Delta\textbf{W}\rVert$} is much larger) when we shift the target frequency components by a smaller distance {\small$\lVert[u_1,v_1]-[u_2,v_2]\rVert$}.

\section{Experiments}\label{sec:exp}

\subsection{Verifying the weakening of high frequencies}\label{sec:exp:b1}

$\bullet$ \textbf{Verifying that a neural network usually learned lowfrequent components first.}
Our theorems prove that a cascaded convolutional decoder network weakens the encoding of high-frequency components. In this experiment, we visualized spectrums of the image generated by a decoder network, which showed that the decoder usually learned low-frequency components in early epochs and then shifted its attention to high-frequency components. To this end, we constructed a cascaded convolutional auto-encoder by using the VGG-16 \citep{vgg} as the encoder network. The decoder network contained four upconvolutional layers. Each convolutional/upconvolutional layer in the auto-encoder applied zero-paddings and was followed by a batch normalization layer and an ReLU layer. The auto-encoder was trained on the Tiny-ImageNet dataset \citep{le2015tiny} using the mean squared error (MSE) loss for image reconstruction\footnote{Please see Appendix~\ref{app:sec:details} for the number of epochs for the training of each model and its fitting error.}. Our theorem was verified by the well-known phenomenon in Figure~\ref{fig:bottleneck}(a), \emph{i.e.}, an auto-encoder usually first generated images with low-frequency components, and then gradually generated more high-frequency components. Results on more datasets in Appendix~\ref{app:sec:exp:lowfirst} yielded similar conclusions.

$\bullet$ \textbf{Verifying that the zero-padding operation strengthened the encoding of low-frequency components.} To this end, we compared feature spectrums between the network with zero-padding operations and the network without zero-padding operations. Therefore, we constructed the following three baseline networks. The first baseline network contained 5 convolutional layers, and each layer applied zero-paddings. Each convolutional layer contained 16 convolutional kernels (kernel size was 7$\times $7), except for the last layer containing 3 convolutional kernels. The second and the third baseline networks were constructed by replacing all zero-padding operations with circular padding operations and replacing all zero-padding operations with mirror padding operations, respectively. Results on the Broden \citep{Bau2017} dataset in Figure~\ref{fig:remark2_3}(c) show that the network with zero-padding operations encoded more significant low-frequency components than the network with circular padding operations. The mirror padding operation also enhanced the significance of low-frequency components, to some extent. Results on more datasets in Appendix~\ref{app:sec:exp:zero} yielded similar conclusions.

\begin{figure*}
	\centering
	\includegraphics[width=0.9\linewidth]{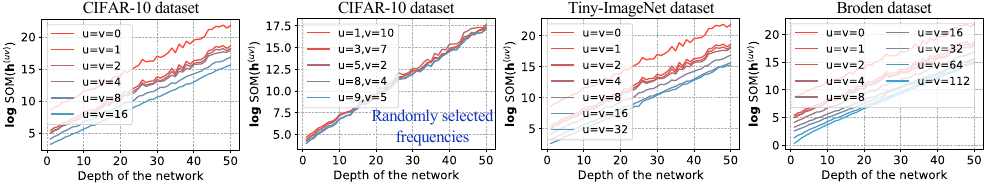}
	\vskip -0.05in
	\caption{The exponential increase of the SOM of feature spectrums, {\small$\textit{SOM}(\textbf{h}^{(uv)})$}, along with the network depth, \emph{i.e.}, the linear increase of {\small$\textbf{\text{log}}\textit{SOM}(\textbf{h}^{(uv)})$} along with the network depth.}
	\label{fig:icml_exp}
\end{figure*}

$\bullet$ \textbf{Verifying factors that strengthened low-frequency
components.} Previous studies \cite{ruderman1994statistics} have empirically found that natural images were dominated by low-frequency components. Therefore, according to Corollaries~\ref{co1_entire} and \ref{co2_back}, we know that if the cascaded convolutional decoder is trained on natural images, then the decoder is more likely to strengthen low-frequencies. Please see Appendix~\ref{app_discussion_natural} for more discussions.	
Besides, we conducted experiments to verify the following three factors that were found to strengthen low frequencies.

\emph{(1) Verifying that a deep network strengthened low-frequency components.}
To this end, we constructed a network with 50 convolutional layers. Each convolutional layer applied zero-paddings to avoid changing the size of feature maps, and was followed by an ReLU layer. We conducted this experiment on three datasets, including CIFAR-10 \citep{krizhevsky2009learning}, Tiny-ImageNet, and Broden datasets, respectively.
The exponential increase of {\small$\boldsymbol{\mathbb{T}}^{(uv)(L:1)}$} along with the network depth {\small$L$} indicated that the frequency component of the network output also increased exponentially along with {\small$L$}.
Therefore, for the frequency component {\small$\textbf{h}^{(uv)}$} generated by
the $l$-th layer in a real decoder network,
we measured its second-order moment {\small$\textit{SOM}(\textbf{h}^{(uv)})$}.
Figure~\ref{fig:icml_exp} shows that {\small$\textit{SOM}(\textbf{h}^{(uv)})$} increased along with the layer number in an exponential manner.

Besides, we visualized feature spectrums of different convolutional layers, which verified the claim that a deep decoder network strengthened the encoding of low-frequency components of the input sample. Results on the Broden dataset in Figure~\ref{fig:remark2_3}(a) show that magnitudes of low frequencies increased along with the network layer number. Results on more datasets in Appendix~\ref{app:sec:exp:factors:L} yielded similar conclusions.

\emph{(2) Verifying that a larger absolute mean value {\small$\mu_l$} of each $l$-th layer's parameters strengthened low-frequency components.} To this end, we compared spectrums of output features, when we set convolution parameters with different mean values {\small$\mu_l$}. Therefore, we applied the network architecture used in the verification of the zero-padding's effects, but we changed the kernel size to 9$\times $9.
Based on this architecture, we constructed three networks, whose parameters were sampled from Gaussian distributions {\small$\mathcal{N}(\mu=0,\sigma^2=0.01^2)$}, {\small$\mathcal{N}(\mu=0.001,\sigma^2=0.01^2)$}, and {\small$\mathcal{N}(\mu=0.01,\sigma^2=0.01^2)$}, respectively. Results on the Broden dataset in Figure~\ref{fig:remark2_3}(b) show that magnitudes of low-frequency components increased along with the absolute mean value of parameters. In addition, Appendix~\ref{app:sec:exp:factors:mu} shows results on more datasets, which also yielded similar conclusions.

We also conducted experiments to measure the effects of large absolute mean values on layers with different depth. Results in Appendix~\ref{app:sec:large_mu_depth} show that no matter which layer had parameters of a large absolute mean value, there was no significant difference in weakening the encoding of high-frequency components. It was because different convolutional layers, including both shallow and deep layers, theoretically had similar roles in affecting the frequency representation of the entire network, according to Corollary~\ref{co1_entire}.

\emph{(3) Verifying that a small kernel size {\small$K$} strengthened low-frequency components.} To this end, we compared feature spectrums of networks with different kernel sizes. Therefore, we constructed three networks with kernel sizes of 1$\times$1, 3$\times$3, and 5$\times$5. Each network contained 5 convolutional layers, each layer contained 16 convolutional kernels, except for the last layer containing 3 kernels. We used the metric {\small$p^{\text{low}}= \frac{\sum_{[u,v]\in\Omega^{\text{low}}}\mathbb{E}_c[\lvert H^{(c)}_{uv}\rvert^2]}{\sum_{uv}\mathbb{E}_c[\lvert H^{(c)}_{uv}\rvert^2]} $} to measure the ratio of low-frequency components to all frequencies, where {\small$\Omega^{\text{low}}$} denoted the set of low-frequency components\footnote{Low-frequencies {\small$[u,v]\in\Omega^{\text{low}}$} were included in {\small$u\in\{u|0 \le u < \frac{M}{8}\} \cup \{u|\frac{7M}{8} \le u < M \}$} and {\small$v\in \{v|0 \le v < \frac{N}{8}  \} \cup \{ v|\frac{7N}{8} \le v < N\}$}.}. Figure~\ref{fig:remark2_3}(d) reports the average {\small$p^{\text{low}}$} value over all images. Results show that the network with a small kernel size encoded more low-frequency components.

\subsection{Verifying the difficulty of fitting shifted frequencies}

This experiment was conducted to verify the difficulty of learning an auto-encoder, when salient frequency components of the target output had a small shift from the spectrum of the input. 
For testing, we constructed an input image {\small$x\sim \mathcal{N}(\textbf{0},\Sigma=0.01\cdot \textbf{I})$} following a Gaussian distribution, which made each frequency component of {\small$x$} have weak strength. We then selected six low frequencies within the range of {\small$u \in \{u | 0 \le u \le 3\} \cup  \{u | 60 < u \le 63\}$} and {\small$v \in \{v | 0 \le v \le 3\} \cup  \{v | 60 < v \le 63\} $}, and we set significant components for these six low frequencies.
Then, we constructed the target image for regression by shifting each salient frequency component {\small$[u,v]$} in {\small$x$} to {\small$[u+\Delta u,v]$} or {\small$[u-\Delta u,v]$} towards higher frequencies\footnote{Please see Appendix~\ref{app:sec:shift} for details about the frequency shift.}. Here, we set {\small$\Delta u=1,4,8$} to generate three target images, respectively.
We trained an auto-encoder with five convolutional layers on each pair of the input image and the target image. Each intermediate convolutional layer contained 64 convolutional kernels (with the kernel size $3\times3$) and followed by an ReLU layer.

We used the metric {\small$\Delta x_{mn} = \big \lVert [x_{mn}^{*(1)},x_{mn}^{*(2)},\ldots,x_{mn}^{*(C)}] - [\hat{x}_{mn}^{(1)},\hat{x}_{mn}^{(2)},\ldots,\hat{x}_{mn}^{(C)}]  \big\rVert_2$} to measure the fitting error on the pixel {\small$[m,n]$} between the target output {\small$x^*\in\mathbb{R}^{C\times M\times N}$} and the network output {\small$\hat{x}\in\mathbb{R}^{C\times M\times N}$}. Figure~\ref{fig:bottleneck}(c) shows the heatmap of the fitting error {\small$ \Delta x_{mn}$}, and the histogram of fitting errors {\small$ \lvert \Delta x_{mn} \rvert $} over different pixels.
Results show that when the target frequency components were shifted by a smaller distance, it was more difficult to learn a decoder to fit the target image, \emph{i.e.,} yielding larger fitting errors in Figure \ref{fig:bottleneck}(c.1,c.2).
Besides, Figure~\ref{fig:bottleneck}(c.3) reports the average weight change {\small$ \lVert \Delta \textbf{W} \rVert$} over different DNNs, which was decreased along with the shifting distance. This also verified that it took more effort to learn a decoder to fit the target image, when the target frequency components were shifted by a smaller distance.

\subsection{Verifying the repeat of certain frequencies}

We conducted experiments to verify the problem that the upsampling operation made a decoder network repeat strong signals at certain frequencies of the generated image in Theorem~\ref{th_upconv}.
To this end, we compared feature spectrums between the input spectrum and the output spectrum of the upsampling layer. We also conducted experiments on the auto-encoder introduced above\footnotemark[6]. Figure~\ref{fig:bottleneck}(b) shows that the decoder network repeated strong signals at certain frequencies of the generated image. Results on more datasets in Appendix~\ref{app:sec:exp:upsample} yielded similar conclusions.

\section{Conclusion}\label{sec:conclusion}
In this paper, we have reformulated the rule for the forward propagation of a cascaded convolutional decoder network in the frequency domain. Based on such propagation rules, we have discovered and theoretically proven that both the convolution operation and the zero-padding operation strengthen low-frequency components in the decoder. The upsampling operation repeats the strong magnitude of the fundamental frequency in the input feature to different frequencies of the spectrum of the output feature map. Besides, we also discover and prove the difficulty of pushing an auto-encoder to fit specific frequency components in the target output, which have a small shift from the spectrum of the input.
Such properties may hurt the representation capacity of a convolutional decoder network. Experiments on ReLU networks have verified our theoretical proofs. 
Note that our findings can explain general trends of networks with above three operations, but cannot derive a deterministic property of a specific network, and cannot be extended to networks for image classification, because we have not derived the property of the max-pooling operation.

\textbf{Acknowledgements.} This work is partially supported
by the National Nature Science Foundation of China (62276165, 62206170), National Key R$\&$D Program of China (2021ZD0111602), Shanghai Natural Science Foundation (21JC1403800,21ZR1434600), National Nature Science Foundation of China (U19B2043), and the Alibaba Group through Alibaba Innovative Research Program.

\bibliography{example_paper}
\bibliographystyle{icml2023}

\newpage
\appendix
\onecolumn

\section{Proofs of our theoretical findings}\label{appendix:theorem}

We first introduce an important equation, which is widely used in the following proofs.

\begin{lemma}\label{le:1}
	\emph{Given $N$ complex numbers, $e^{i n\theta}$, $n=0,1,\ldots,N-1$, the sum of these $N$ complex numbers is given as follows.
		\begin{equation}\label{eq:sum_of_e}
			\forall \theta \in \mathbb{R},\qquad \sum_{n=0}^{N-1}e^{i n\theta} = \frac{\sin(\frac{N \theta}{2})}{\sin(\frac{\theta}{2})}e^{i\frac{(N-1)\theta}{2}} 
		\end{equation}
		Specifically, when $N\theta = 2k\pi , k \in \mathbb{Z}$, $-N < k < N$, we have
		\begin{equation}\label{eq:delta}
			\begin{aligned}
				\forall \theta \in \mathbb{R},\quad &\sum_{n=0}^{N-1}e^{i n\theta} = \frac{\sin(\frac{N \theta}{2})}{\sin(\frac{\theta}{2})}e^{i\frac{(N-1)\theta}{2}} = N \delta_{\theta} ;\quad \text{s.t.} \ N\theta = 2k\pi , k \in \mathbb{Z} ,-N < k < N, \\
				&\textit{where}\quad	\delta_{\theta}=  \begin{cases}
					1,\ &\theta=0 \\ 0,\ & \text{otherwise}\end{cases}
			\end{aligned}
	\end{equation}}
\end{lemma}

We prove Lemma~\ref{le:1} as follows.

\begin{proof}
	First, let us use the letter $S\in \mathbb{C}$ to denote the term of $\sum_{n=0}^{N-1}e^{i n\theta} $.
	\begin{equation*}
		S = \sum_{n=0}^{N-1}e^{i n \theta } 
	\end{equation*}
	
	Therefore, $e^{i\theta }S$ is formulated as follows.
	\begin{equation*}
		e^{i\theta }S = \sum_{n=1}^{N}e^{in\theta } \in \mathbb{C}
	\end{equation*}
	
	Then, $S$ can be computed as $S= \frac{e^{i\theta }S-S}{e^{i\theta }-1}$. Therefore, we have
	\begin{equation*}
		\begin{aligned}
			S & = \frac{e^{i\theta }S-S}{e^{i\theta }-1} \\
			& = \frac{\sum_{n=1}^{N}e^{in\theta } -\sum_{n=0}^{N-1}e^{i n \theta } }{e^{i\theta }-1} \\
			&= \frac{e^{iN\theta } -1 }{e^{i\theta} -1}\\
			&= \frac{e^{i\frac{N\theta }{2}}- e^{-i\frac{N\theta }{2}}}{e^{i\frac{\theta}{2}}- e^{-i\frac{\theta}{2}}} e^{i\frac{(N-1)\theta}{2}} \\
			&= \frac{(e^{i\frac{N\theta }{2}}- e^{-i\frac{N\theta }{2}})/2i}{(e^{i\frac{\theta}{2}}- e^{-i\frac{\theta}{2}})/2i} e^{i\frac{(N-1)\theta}{2}} \\
			&=\frac{\sin(\frac{N \theta}{2})}{\sin(\frac{\theta}{2})}e^{i\frac{(N-1)\theta}{2}}
		\end{aligned}
	\end{equation*}
	Therefore, we prove that $\sum_{n=0}^{N-1}e^{i n \theta } =\frac{\sin(\frac{N \theta}{2})}{\sin(\frac{\theta}{2})}e^{i\frac{(N-1)\theta}{2}}$. 
	
	Then, we prove the special case that when $N\theta = 2k\pi , k \in \mathbb{Z} ,-N < k < N$, $\sum_{n=0}^{N-1}e^{i n\theta} = N \delta_{\theta}=\begin{cases}
		N,\ &\theta=0 \\ 0,\ & \text{otherwise}\end{cases}$, as follows.
	
	When $\theta = 0$, we have
	\begin{equation*}
		\begin{aligned}
			\lim_{\theta \to 0}  \sum_{n=0}^{N-1} e^{in\theta} 
			&= \lim_{\theta \to 0}  \frac{\sin(\frac{N\theta}{2})}{\sin(\frac{\theta}{2})}e^{i\frac{(N-1)\theta}{2}} \\
			&= \lim_{\theta \to 0}  \frac{\sin(\frac{N\theta}{2})}{\sin(\frac{\theta}{2})} \\  
			&= N
		\end{aligned}
	\end{equation*}
	
	When $\theta \ne 0$, and $N\theta = 2k\pi,k \in \mathbb{Z} ,-N < k < N$, we have
	\begin{equation*}
		\begin{aligned}
			\sum_{n=0}^{N-1} e^{in\theta} 
			&= \frac{\sin(\frac{N\theta}{2})}{\sin(\frac{\theta}{2})}e^{i\frac{(N-1)\theta}{2}} \\
			&= \frac{\sin(k\pi)}{\sin(\frac{k\pi}{N})}e^{i\frac{(N-1)k\pi}{N}} \\
			&= 0 \\
		\end{aligned}
	\end{equation*}
	
\end{proof}

In the following proofs, the following two equations are widely used, which are derived based on Lemma~\ref{le:1}.

\begin{equation*}
	\begin{aligned}
		\sum_{m=0}^{M-1}\sum_{n=0}^{N-1} e^{-i(\frac{um}{M}+\frac{vn}{N})2\pi} 
		&= \sum_{m=0}^{M-1}  e^{im(-\frac{u2\pi}{M})} \sum_{n=0}^{N-1}  e^{in(-\frac{v2\pi}{N})}
		\\ 
		& = (M\delta_{-\frac{u2\pi}{M}}) (N\delta_{-\frac{v2\pi}{N}})
		\quad {\rm \slash \slash According\ to\ Equation~(\ref{eq:delta}})\\
		& = \begin{cases}
			MN,\ &u=v=0 \\ 0,\ & \text{otherwise}
		\end{cases}
	\end{aligned}
\end{equation*}

To simplify the representation, \textbf{let $\delta_{uv}$ be the simplification of  $\delta_{-\frac{u2\pi}{M}}\delta_{-\frac{v2\pi}{N}}$ in the following proofs.} Therefore, we have 
\begin{equation}\label{eq:delta_uv}
	\sum_{m=0}^{M-1}\sum_{n=0}^{N-1} e^{-i(\frac{um}{M}+\frac{vn}{N})2\pi} = MN\delta_{uv} = \begin{cases}
		MN,\ &u=v=0 \\ 0,\ & \text{otherwise}
	\end{cases}
\end{equation}

Similarly, we derive the second equation as follows.

\begin{small}\begin{equation}\label{eq:delta_u_v_}
		\begin{aligned}
			\sum_{m=0}^{M-1}\sum_{n=0}^{N-1} e^{i(\frac{(u-u')m}{M}+\frac{(v-v')n}{N})2\pi} 
			&= \sum_{m=0}^{M-1}  e^{im(\frac{(u-u')2\pi}{M})} \sum_{n=0}^{N-1}  e^{in(\frac{(v-v')2\pi}{N})} \\
			& = MN\delta_{\frac{(u-u')2\pi}{M}}\delta_{\frac{(v-v')2\pi}{N}}
			\quad {\rm \slash \slash According\ to\ Equation~(\ref{eq:delta}})\\
			& = MN\delta_{u-u'}\delta_{v-v'}\\
			& = \begin{cases}
				MN,\ &u'=u;v'=v \\ 0,\ & \text{otherwise}
			\end{cases}
		\end{aligned}
\end{equation}\end{small}

\subsection{Proof of Theorem~\ref{th2_layerwise}}\label{app:sec:th2}
In this section, we prove Theorem \ref{th2_layerwise} in Section~\ref{sec:dynamics} of the main paper, as follows.

\begin{proof}
	
	Given each $c$-th channel of the feature spectrum $G^{(c)}$, the corresponding feature $F^{(c)}$ in the time domain can be computed as follows.
	\begin{equation*}
		F^{(c)}_{mn} = \frac{1}{MN}\sum_{u=0}^{M-1}\sum_{v=0}^{N-1} G^{(c)}_{uv} e^{i(\frac{um}{M}+\frac{vn}{N})2\pi} 
	\end{equation*}
	
	Then, let us conduct the convlution operation (in Equation~(\ref{eq:layerwise_conv}) in the main paper) on feature $\textbf{F}=[F^{(1)},F^{(2)},\ldots,F^{(C)}]$, in order to obtain the output feature $\widetilde{\textbf{F}}\in\mathbb{R}^{D\times M' \times N'}$.
	\begin{small}\begin{equation*}
			\begin{aligned}
				\forall d=1,2,\ldots,&D; 0\le m< M';0\le n<N'; \quad\\
				\tilde{F}^{(d)}_{mn} &=b^{(d)} + \sum_{c=1}^{C} \sum_{t=0}^{K-1}\sum_{s=0}^{K-1}W_{cts}^{ker=d} F^{(c)}_{m+t,n+s}\\ 
				&= b^{(d)} +\sum_{c=1}^{C}   \sum_{t=0}^{K-1}\sum_{s=0}^{K-1}W_{cts}^{ker=d} \frac{1}{MN} \sum_{u=0}^{M-1}\sum_{v=0}^{N-1} G^{(c)}_{uv}  e^{i(\frac{u(m+t)}{M}+\frac{v(n+s)}{N})2\pi}\\
				&= b^{(d)} +\sum_{c=1}^{C} \frac{1}{MN}  \sum_{u=0}^{M-1}\sum_{v=0}^{N-1} G^{(c)}_{uv} e^{i(\frac{um}{M}+\frac{vn}{N})2\pi}\sum_{t=0}^{K-1}\sum_{s=0}^{K-1}W_{cts}^{ker=d}e^{i(\frac{ut}{M}+\frac{vs}{N})2\pi} \\
				&= b^{(d)} +\sum_{c=1}^{C} \frac{1}{MN}  \sum_{u=0}^{M-1}\sum_{v=0}^{N-1}T_{dc}^{(uv)} G^{(c)}_{uv} e^{i(\frac{um}{M}+\frac{vn}{N})2\pi} \\
			\end{aligned}
	\end{equation*}\end{small}
	
	Then, let us conduct the DFT on each channel of $\widetilde{\textbf{F}}$, in order to obtain feature spectrums $H_{u'v'}^{(d)}$ of $\widetilde{\textbf{F}}$.
	\begin{small}\begin{equation*}
			\begin{aligned}
				\forall d=1,2,&\ldots,D;\ 0\le u'< M';\ 0\le v'<N'; \quad \\
				H_{u'v'}^{(d)} &= \sum_{m=0}^{M'-1}\sum_{n=0}^{N'-1} \widetilde{F}^{(l,d)}_{mn} e^{-i(\frac{u'm}{M'}+\frac{v'n}{N'})2\pi} \\
				&=\sum_{m=0}^{M'-1}\sum_{n=0}^{N'-1} e^{-i(\frac{u'm}{M'}+\frac{v'n}{N'})2\pi} (b^{(d)} +\sum_{c=1}^{C} \frac{1}{MN}  \sum_{u=0}^{M-1}\sum_{v=0}^{N-1} T_{dc}^{(uv)} G^{(c)}_{uv} e^{i(\frac{um}{M}+\frac{vn}{N})2\pi})
				\quad {\rm \slash \slash Equation~(\ref{eq:delta_uv})} \\ 
				&= M'N' b^{(d)} \delta_{u'v'} + \sum_{c=1}^{C}  \sum_{u=0}^{M-1}\sum_{v=0}^{N-1} T_{dc}^{(uv)} G^{(c)}_{uv} \frac{1}{MN} \sum_{m=0}^{M'-1}\sum_{n=0}^{N'-1} e^{i((\frac{u}{M}-\frac{u'}{M'})m+(\frac{v}{N}-\frac{v'}{N'})n)2\pi} \\
				&\slash \slash \text{ Let} \ \alpha_{u'v'uv} = \frac{1}{MN} \sum_{m=0}^{M'-1}\sum_{n=0}^{N'-1} e^{i((\frac{u}{M}-\frac{u'}{M'})m+(\frac{v}{N}-\frac{v'}{N'})n)2\pi} \\
				&= M'N' b^{(d)} \delta_{u'v'} +\sum_{u=0}^{M-1}\sum_{v=0}^{N-1} \alpha_{u'v'uv} \sum_{c=1}^{C}   T_{dc}^{(uv)} G^{(c)}_{uv} 
				\quad 
			\end{aligned}
	\end{equation*}\end{small}
	
	When the convlution operation does not apply paddings, and its stride size is 1, $M'=M-K+1,N'=N-K+1$. In this way, $\alpha_{u'v'uv}$ can be rewritten as follows.
	\begin{small}
		\begin{equation}\label{app:eq:alpha}
			\begin{aligned}
				\alpha_{u'v'uv} &= \frac{1}{MN}\sum_{m=0}^{M'-1}\sum_{n=0}^{N'-1} e^{i((\frac{u}{M}-\frac{u'}{M'})m+(\frac{v}{N}-\frac{v'}{N'})n)2\pi} \\
				& \slash \slash M'=M-K+1,N'=N-K+1\\
				&= \frac{1}{MN}\sum_{m=0}^{M-K}\sum_{n=0}^{N-K} e^{i((\frac{u}{M}-\frac{u'}{M-K+1})m+(\frac{v}{N}-\frac{v'}{N-K+1})n)2\pi} \\
				&= \frac{1}{MN}\sum_{m=0}^{M-K}e^{i(\frac{u}{M}-\frac{u'}{M-K+1})2\pi m} \sum_{n=0}^{N-K} e^{i(\frac{v}{N}-\frac{v'}{N-K+1})2\pi n } \\
				&  {\rm \slash \slash According\ to\ Equation~(\ref{eq:sum_of_e})} \\
				&= \frac{1}{MN} \frac{\sin((M-K)\lambda_{uu'}\pi)}{\sin(\lambda_{uu'}\pi)}\frac{\sin((N-K)\gamma_{vv'}\pi)}{\sin(\gamma_{vv'}\pi)}e^{i((M-K)\lambda_{uu'}+(N-K)\gamma_{vv'})\pi}  \\
			\end{aligned}
		\end{equation}
	\end{small}
	where $\lambda_{uu'} = \frac{(u-u')M-u(K-1)}{M(M-K+1)}$,
	$\gamma_{vv'} = \frac{(v-v')N-v(K-1)}{N(N-K+1)}$.
	
	Therefore, we prove that the vector $\textbf{h}^{(u'v')} = [H_{u'v'}^{(1)},H_{u'v'}^{(2)},\ldots,H_{u'v'}^{(D)}]^{\top}\in\mathbb{C}^D$ can be computed as follows.
	\begin{equation*}
		\forall d=1,2,\ldots,D; \quad \textbf{h}^{(u'v')} = \delta_{u'v'} M'N' \textbf{b} + \sum_{u=0}^{M-1}\sum_{v=0}^{N-1} \alpha_{u'v'uv}  T^{(uv)} \textbf{g}^{(uv)}  \\
	\end{equation*}

	Furthermore, based on Assumption~\ref{ass1}, the convolution operation does not change the size of the feature map, \emph{i.e.}, $M'=M$, $N'=N$. In this case, $\alpha_{u'v'uv}$ can be computed as follows.
	
	\begin{equation}\label{eq:alpha-delta}
		\begin{aligned}
			\alpha_{u'v'uv} &= \frac{1}{MN}\sum_{m=0}^{M'-1}\sum_{n=0}^{N'-1} e^{i((\frac{u}{M}-\frac{u'}{M'})m+(\frac{v}{N}-\frac{v'}{N'})n)2\pi} \\
			&= \frac{1}{MN}\sum_{m=0}^{M-1}\sum_{n=0}^{N-1} e^{i(\frac{(u-u')m}{M}+\frac{(v-v')n}{N})2\pi} 
			\quad \slash \slash M'=M,N'=N\\
			&= \frac{1}{MN}\sum_{m=0}^{M-1}e^{i(\frac{(u-u')2\pi}{M})m} \sum_{n=0}^{N-1}e^{i(\frac{(v-v')2\pi}{N})n} 
			\quad {\rm \slash \slash According\ to\ Equation~(\ref{eq:delta_u_v_})}\\
			&= \delta_{u-u'}\delta_{v-v'}
		\end{aligned}
	\end{equation}
	where\ $ \delta_{u-u'} = \left\{\begin{aligned} 1,\ &u'=u \\ 0,\ & \text{otherwise}\end{aligned} \right.$; 
	$ \delta_{v-v'} = \left\{\begin{aligned} 1,\ &v'=v \\ 0,\ & \text{otherwise}\end{aligned} \right.$.
	
	Therefore, $\textbf{h}^{(u'v')}$ can be computed as follows.
	\begin{equation*}
		\begin{aligned}
			\textbf{h}^{(u'v')} &= \sum_{u=0}^{M'-1}\sum_{v=0}^{N'-1} \alpha_{u'v'uv}  T^{(u'v')} \textbf{g}^{(u'v')}+ \delta_{u'v'}  M'N' \textbf{b}  \\
			&=\sum_{u=0}^{M-1}\sum_{v=0}^{N-1} \delta_{u-u'}\delta_{v-v'}  T^{(u'v')} \textbf{g}^{(u'v')}+ \delta_{u'v'} MN \textbf{b}  \\
			&=T^{(u'v')} \textbf{g}^{(u'v')} + MN \textbf{b} \delta_{u'v'}
		\end{aligned}
	\end{equation*}
	
	Then, we prove that $\textbf{h}^{(uv)}=T^{(uv)} \textbf{g}^{(uv)} + MN \textbf{b} \delta_{uv}$.
\end{proof}

\subsection{Proof of Corollary~\ref{co1_entire}}\label{app:sec:co1}
In this section, we prove Corollary~\ref{co1_entire} in Section~\ref{sec:dynamics} of the main paper, as follows.

\begin{proof}
	Let $\textbf{G}^{(l)} = [G^{(l,1)},G^{(l,2)},\cdots,G^{(l,C_{l})}] \in \mathbb{C}^{C_{l} \times M \times N}$ denote feature spectrums of the $l$-th layer. Let $\textbf{g}^{(l,uv)} = [G^{(l,1)}_{uv},G^{(l,2)}_{uv},\cdots,G^{(l,C_{l})}_{uv}]^{\top} \in \mathbb{C}^{C_{l}} $ denote the frequency component at the frequency $[u,v]$. When $l=0$, $\textbf{g}^{(0,uv)}$ denotes the frequency component of the input sample. When $l=L$, $\textbf{g}^{(L,uv)}$ denotes the frequency component of the network output.
	
	Based on Theorem~\ref{th2_layerwise}, $\textbf{g}^{(l,uv)}$ can be computed as follows.
	\begin{equation*}
		\begin{aligned}
			\forall l=1,2,\ldots,L, \quad \textbf{g}^{(l,uv)} 
			&= T^{(l,uv)} \textbf{g}^{(l-1,uv)} + \delta_{uv} MN \textbf{b}^{(l)} 
		\end{aligned}
	\end{equation*}
	
	Then, the frequency component $\textbf{g}^{(L,uv)}$ of the network output can be computed as follows.
	\begin{small}
		\begin{equation*}
			\begin{aligned}
				\textbf{g}^{(L,uv)} 
				&= T^{(L,uv)} \textbf{g}^{(L-1,uv)} + \delta_{uv} MN \textbf{b}^{(L)}  \\
				&= T^{(L,uv)} (T^{(L-1,uv)} \textbf{g}^{(L-2,uv)} + \delta_{uv} MN \textbf{b}^{(L-1)}) + \delta_{uv} MN \textbf{b}^{(L)}  \\
				&= T^{(L,uv)} T^{(L-1,uv)} \textbf{g}^{(L-2,uv)} + T^{(L,uv)} \delta_{uv} MN \textbf{b}^{(L-1)} + \delta_{uv} MN \textbf{b}^{(L)}  \\
				&= \cdots \\
				&= T^{(l,uv)}_{dc}\cdots T^{(1,uv)} \textbf{g}^{(0,uv)} + MN T^{(l,uv)}_{dc}\cdots T^{(2,uv)} \textbf{b}^{(1)}\delta_{uv} + \cdots + MN \textbf{b}^{(L)}\delta_{uv}\\
				&= T^{(l,uv)}_{dc}\cdots T^{(1,uv)} \textbf{g}^{(0,uv)} + \delta_{uv} MN ( T^{(l,uv)}_{dc}\cdots T^{(2,uv)} \textbf{b}^{(1)} + \cdots + MN \textbf{b}^{(L)})
			\end{aligned}
		\end{equation*}
	\end{small}
	
	Let $ \mathbb{T}^{(uv)(L:1)} = T^{(l,uv)}_{dc}\cdots T^{(2,uv)} T^{(1,uv)}$ and $\boldsymbol{\beta}=MN\big(\textbf{b}^{(L)}+\sum_{j=2}^{L}\boldsymbol{\mathbb{T}}^{(00)(L:j)}\textbf{b}^{(j-1)} \big)$. Let $\textbf{h}^{(uv)}=\textbf{g}^{(L,uv)}$ denote the frequency component of the network output, and let $\textbf{g}^{(uv)}=\textbf{g}^{(0,uv)}$ denote the frequency component of the input sample. Then, we prove that $\textbf{h}^{(uv)}$ can be computed as follows.
	\begin{equation*}
		\textbf{h}^{(uv)} =\mathbb{T}^{(uv)(L:1)} \textbf{g}^{(uv)}+ \delta_{uv}\boldsymbol{\beta}
	\end{equation*}
	
\end{proof}

\subsection{Proof of Corollary~\ref{co2_back}}\label{app:sec:co2}
In this section, we prove Corollary~\ref{co2_back} in Section~\ref{sec:dynamics} of the main paper, as follows.

\begin{proof}
	
	\textbf{First, we focus on a single convolutional layer.}
	
	According to the DFT and the inverse DFT, we can obtain the mathematical relationship between $G_{uv}^{(l,c)}$ and $F_{mn}^{(l,c)}$, and the mathematical relationship between $T_{dc}^{(l,uv)}$ and $W_{cts}^{(l)[\text{ker}=d]}$, as follows.
	\begin{small}\begin{equation}\label{eq:dft}
			\begin{aligned}
				\left\{
				\begin{aligned}
					& G^{(l,c)}_{uv} = \sum_{m=0}^{M-1}\sum_{n=0}^{N-1} F^{(l,c)}_{mn}e^{-i(\frac{um}{M}+\frac{vn}{N})2\pi}  \\
					& F^{(l,c)}_{mn} = \frac{1}{MN}\sum_{u=0}^{M-1}\sum_{v=0}^{N-1}G^{(l,c)}_{uv}e^{i(\frac{um}{M}+\frac{vn}{N})2\pi} 
				\end{aligned}
				\right. 
				\left\{
				\begin{aligned}
					& T^{(l,uv)}_{dc} = \sum_{t=0}^{K-1}\sum_{s=0}^{K-1} W_{cts}^{(l)[\text{ker}=d]}e^{i(\frac{ut}{M}+\frac{vs}{N})2\pi} \\
					& W_{cts}^{(l)[\text{ker}=d]} = \frac{1}{MN}\sum_{u=0}^{M-1}\sum_{v=0}^{N-1}T^{(l,uv)}_{dc}e^{-i(\frac{ut}{M}+\frac{vs}{N})2\pi} 
				\end{aligned}
				\right.
			\end{aligned}
	\end{equation}\end{small}
	
	Based on Equation~(\ref{eq:dft}) and the derivation rule for complex numbers \citep{kreutz2009complex}, we can obtain the mathematical relationship between $\frac{\partial \textit{Loss}}{\partial \overline{G}_{uv}^{(l,c)}}$ and $\frac{\partial \textit{Loss}}{\partial \overline{F}_{mn}^{(l,c)}}$, and the mathematical relationship between $\frac{\partial \textit{Loss}}{\partial \overline{T}_{dc}^{(l,uv)}}$ and $\frac{\partial \textit{Loss}}{\partial\overline{ W}_{cts}^{(l)[\text{ker}=d]} }$, as follows. Note that when we use gradient descent to optimize a real-valued loss function $\textit{Loss}$ with complex variables, people usually treat the real and imaginary values, $a\in\mathbb{C}$ and $b\in\mathbb{C}$, of a complex variable ($z=a+bi$) as two separate real-valued variables, and separately update these two real-valued variables. In this way, the exact optimization step of $z$ computed based on such a technology is equivalent to $\frac{\partial \textit{Loss}}{\partial \overline{z}}$. Since $F^{(l,c)}_{mn}$ and $W_{cts}^{(l)[\text{ker=d}]}$ are real numbers, $\frac{\partial Loss}{\partial \overline{F}^{(l,c)}_{mn}} =	\frac{\partial Loss}{\partial F^{(l,c)}_{mn}}$ and $\frac{\partial Loss}{\partial \overline{W}_{cts}^{(l)[\text{ker=d}]}} =	\frac{\partial Loss}{\partial W_{cts}^{(l)[\text{ker=d}]}}$.
	
	\begin{small}\begin{equation}\label{eq:dft_grad}
			\begin{aligned}
				\left\{
				\begin{aligned}
					& \frac{\partial Loss}{\partial \overline{G}^{(l,c)}_{uv}} =\frac{1}{MN}\sum_{m=0}^{M-1}\sum_{n=0}^{N-1}\frac{\partial Loss}{\partial \overline{F}^{(l,c)}_{mn}}e^{-i(\frac{um}{M}+\frac{vn}{N})2\pi}  \\
					& \frac{\partial Loss}{\partial \overline{F}^{(l,c)}_{mn}} =\sum_{u=0}^{M-1}\sum_{v=0}^{N-1} \frac{\partial Loss}{\partial \overline{G}^{(l,c)}_{uv}}e^{i(\frac{um}{M}+\frac{vn}{N})2\pi}
				\end{aligned}
				\right.
				\left\{
				\begin{aligned}
					& \frac{\partial Loss}{\partial \overline{T}^{(l,uv)}_{dc} }= \frac{1}{MN}\sum_{t=0}^{K-1}\sum_{s=0}^{K-1} \frac{\partial Loss}{\partial  \overline{W}_{cts}^{(l)[\text{ker=d}]}}e^{i(\frac{ut}{M}+\frac{vs}{N})2\pi}  \\
					& \frac{\partial Loss}{\partial \overline{W}_{cts}^{(l)[\text{ker=d}]}} = \sum_{u=0}^{M-1}\sum_{v=0}^{N-1}\frac{\partial Loss}{\partial \overline{T}^{(l,uv)}_{dc}}e^{-i(\frac{ut}{M}+\frac{vs}{N})2\pi} 
				\end{aligned}
				\right.
			\end{aligned}
	\end{equation}\end{small}

	Let us conduct the convolution operation (based on Assumption~\ref{ass1}) on the feature map $\textbf{F}^{(l-1)}=[F^{(l-1,1)},F^{(l-1,2)},\ldots,F^{(l-1,C)}]\in\mathbb{R}^{C\times M\times N}$, and obtain the output feature map $\textbf{F}^{(l)}=[F^{(l,1)},F^{(l,2)},\ldots,F^{(l,D)}]\in\mathbb{R}^{D\times M\times N}$ of the l-th layer as follows.
	\begin{equation}\label{app_eq:conv}
		F^{(l,d)}_{mn} =b^{(d)} + \sum_{c=1}^{C} \sum_{t=0}^{K-1}\sum_{s=0}^{K-1}W_{cts}^{(l)[\text{ker}=d]} F^{(l-1,c)}_{m+t,n+s}\\ 
	\end{equation}
	
	Based on Equation~(\ref{eq:dft}) and Equation~(\ref{eq:dft_grad}), and the derivation rule for complex numbers \citep{kreutz2009complex}, the exact optimization step of $T^{(l,uv)}_{dc}$ in real implementations can be computed as follows.
	\begin{small}\begin{equation*}
			\begin{aligned}
				&\frac{\partial \textit{Loss}}{\partial \overline{T}^{(l,uv)}_{dc} } \\
				&= \frac{1}{MN}\sum_{t=0}^{K-1}\sum_{s=0}^{K-1} \frac{\partial \textit{Loss}}{\partial \overline{W}_{cts}^{(l)[\text{ker}=d]}}e^{i(\frac{ut}{M}+\frac{vs}{N})2\pi} 
				\quad {\rm \slash \slash Equation~(\ref{eq:dft_grad})}\\
				& =\frac{1}{MN}\sum_{t=0}^{K-1}\sum_{s=0}^{K-1} \left( \sum_{m=0}^{M-1}\sum_{n=0}^{N-1} \frac{\partial \textit{Loss}}{\partial \overline{F}^{(l,d)}_{mn}}\cdot \overline{F}^{(l-1,c)}_{m+t,n+s}\right)e^{i(\frac{ut}{M}+\frac{vs}{N})2\pi} 
				\quad {\rm \slash \slash Equation~(\ref{app_eq:conv})}\\
				&  \quad {\rm \slash \slash Equation~(\ref{eq:dft})} \\
				& = \frac{1}{MN}\sum_{t=0}^{K-1}\sum_{s=0}^{K-1} \left(\sum_{m=0}^{M-1}\sum_{n=0}^{N-1} \frac{\partial \textit{Loss}}{\partial \overline{F}^{(l,d)}_{mn}}\cdot  \frac{1}{MN}\sum_{u'=0}^{M-1}\sum_{v'=0}^{N-1} \overline{G}^{(l-1,c)}_{u'v'} e^{-i(\frac{u'(m+t)}{M}+\frac{v'(n+s)}{N})2\pi} \right)e^{i(\frac{ut}{M}+\frac{vs}{N})2\pi}\\
				& = \frac{1}{MN}\sum_{t=0}^{K-1}\sum_{s=0}^{K-1} \left( \sum_{u'=0}^{M-1}\sum_{v'=0}^{N-1} \overline{G}^{(l-1,c)}_{u'v'} e^{-i(\frac{u't}{M}+\frac{v's}{N})2\pi} \cdot \frac{1}{MN}\sum_{m=0}^{M-1}\sum_{n=0}^{N-1}\frac{\partial \textit{Loss}}{\partial \overline{F}^{(l,d)}_{mn}} e^{-i(\frac{u'm}{M}+\frac{v'n}{N})2\pi} \right)e^{i(\frac{ut}{M}+\frac{vs}{N})2\pi} \\
				& = \frac{1}{MN}\sum_{t=0}^{K-1}\sum_{s=0}^{K-1} \left( \sum_{u'=0}^{M-1}\sum_{v'=0}^{N-1}\overline{G}^{(l-1,c)}_{u'v'}  \frac{\partial Loss}{\partial \overline{G}^{(l,d)}_{u'v'}}e^{-i(\frac{u't}{M}+\frac{v's}{N})2\pi} \right)e^{i(\frac{ut}{M}+\frac{vs}{N})2\pi} \quad {\rm \slash \slash Equation~(\ref{eq:dft_grad})} \\
				&= \frac{1}{MN}\sum_{t=0}^{K-1}\sum_{s=0}^{K-1} \sum_{u'=0}^{M-1}\sum_{v'=0}^{N-1}\overline{G}^{(l-1,c)}_{u'v'}  \frac{\partial \textit{Loss}}{\partial \overline{G}^{(l,d)}_{u'v'}} e^{i(\frac{(u-u')t}{M}+\frac{(v-v')s}{N})2\pi}\\
				&=  \sum_{u'=0}^{M-1}\sum_{v'=0}^{N-1}\overline{G}^{(l-1,c)}_{u'v'}  \frac{\partial \textit{Loss}}{\partial \overline{G}^{(l,d)}_{u'v'}} \cdot \frac{1}{MN}\sum_{t=0}^{K-1}\sum_{s=0}^{K-1}  e^{i(\frac{(u-u')t}{M}+\frac{(v-v')s}{N})2\pi} \\
				& \slash \slash \text{ Let} \ \chi_{u'v'uv} = \frac{1}{MN}\sum_{t=0}^{K-1}\sum_{s=0}^{K-1}  e^{i(\frac{(u-u')t}{M}+\frac{(v-v')s}{N})2\pi} \\
				&=  \sum_{u'=0}^{M-1} \sum_{v'=0}^{N-1} \chi_{u'v'uv}  \overline{G}^{(l-1,c)}_{u'v'}  \frac{\partial \textit{Loss}}{\partial \overline{G}^{(l,d)}_{u'v'}} 
			\end{aligned}
	\end{equation*}\end{small}
	where $\chi_{u'v'uv}$ can be rewritten as follows.
	\begin{small}\begin{equation*}
			\begin{aligned}
				\chi_{u'v'uv}
				&= \frac{1}{MN}\sum_{t=0}^{K-1}\sum_{s=0}^{K-1}  e^{i(\frac{(u-u')t}{M}+\frac{(v-v')s}{N})2\pi}\\
				&= \frac{1}{MN}\sum_{t=0}^{K-1}e^{i\frac{(u-u')2\pi}{M} t} \sum_{s=0}^{K-1} e^{i\frac{(v-v')2\pi}{N} s}\\
				&=  \frac{1}{MN} \frac{\sin(\frac{K(u-u')\pi}{M})}{\sin(\frac{(u-u')\pi}{M})}  \frac{\sin(\frac{K(v-v')\pi}{N})}{\sin(\frac{(v-v')\pi}{N})} \cdot e^{i(\frac{(K-1)(u-u')}{M}+\frac{(K-1)(v-v')}{N})\pi}  \quad {\rm \slash \slash According\ to\ Equation\ (\ref{eq:sum_of_e})} \\
			\end{aligned}
	\end{equation*}\end{small}

	Similarly, we computed the gradient of the loss function \emph{w.r.t.} the spectrum map $\overline{G}^{(l-1,c)}$ as follows.
	\begin{small}\begin{equation*}
			\begin{aligned}
				&\frac{\partial Loss}{\partial \overline{G}^{(l-1,c)}_{u'v'}} \\
				&= \frac{1}{MN}\sum_{m=0}^{M-1}\sum_{n=0}^{N-1}\frac{\partial Loss}{\partial\overline{F}^{(l-1,c)}_{mn}}e^{-i(\frac{u'm}{M}+\frac{v'n}{N})2\pi}
				\quad {\rm \slash \slash Equation~(\ref{eq:dft_grad})}\\
				&= \frac{1}{MN}\sum_{m=0}^{M-1}\sum_{n=0}^{N-1} \left( \sum_{t=0}^{K-1}\sum_{s=0}^{K-1} \overline{W}_{cts}^{(l)[\text{ker}=d]} \cdot \frac{\partial Loss}{\partial \overline{F}^{(l,d)}_{m-t,n-s}} \right)e^{-i(\frac{u'm}{M}+\frac{v'n}{N})2\pi} \quad {\rm \slash \slash Equation~(\ref{app_eq:conv})}\\
				& {\rm \slash \slash According \ to \ Equation~(\ref{eq:dft_grad})} \\
				&= \frac{1}{MN}\sum_{m=0}^{M-1}\sum_{n=0}^{N-1} \left(\sum_{t=0}^{K-1}\sum_{s=0}^{K-1} \overline{W}_{cts}^{(l)[\text{ker}=d]} \cdot\sum_{u=0}^{M-1}\sum_{v=0}^{N-1} \frac{\partial Loss}{\partial \overline{G}^{(l,d)}_{uv}} e^{i(\frac{u(m-t)}{M}+\frac{v(n-s)}{N})2\pi} \right)e^{-i(\frac{u'm}{M}+\frac{v'n}{N})2\pi}\\
				&= \frac{1}{MN}\sum_{m=0}^{M-1}\sum_{n=0}^{N-1} \left(\sum_{u=0}^{M-1}\sum_{v=0}^{N-1} \frac{\partial Loss}{\partial \overline{G}^{(l,d)}_{uv}}e^{i(\frac{um}{M}+\frac{vn}{N})2\pi} \cdot \sum_{t=0}^{K-1}\sum_{s=0}^{K-1} \overline{W}_{cts}^{(l)[\text{ker}=d]}e^{-i(\frac{ut}{M}+\frac{vs}{N})2\pi} \right) e^{-i(\frac{u'm}{M}+\frac{v'n}{N})2\pi}\\
				&= \frac{1}{MN}\sum_{m=0}^{M-1}\sum_{n=0}^{N-1} \left( \sum_{u=0}^{M-1}\sum_{v=0}^{N-1} \frac{\partial Loss}{\partial \overline{G}^{(l,d)}_{uv}}\overline{T}^{(l,uv)}_{dc}e^{i(\frac{um}{M}+\frac{vn}{N})2\pi} \right)e^{-i(\frac{u'm}{M}+\frac{v'n}{N})2\pi}\quad {\rm \slash \slash Equation~(\ref{eq:dft})}\\
				&= \sum_{u=0}^{M-1}\sum_{v=0}^{N-1} \frac{\partial Loss}{\partial \overline{G}^{(l,d)}_{uv}}\overline{T}^{(l,uv)}_{dc} \cdot \frac{1}{MN}\sum_{m=0}^{M-1}\sum_{n=0}^{N-1}e^{i(\frac{(u-u')m}{M}+\frac{(v-v')n}{N})2\pi}\\
				&= \sum_{u=0}^{M-1}\sum_{v=0}^{N-1} \frac{\partial Loss}{\partial \overline{G}^{(l,d)}_{uv}}\overline{T}^{(l,uv)}_{dc} \cdot \delta_{u-u'}\delta_{v-v'}
				\quad {\rm \slash \slash Equation~(\ref{eq:delta_u_v_})}\\
				&=  \frac{\partial Loss}{\partial \overline{G}^{(l,d)}_{u'v'}}\overline{T}^{(l,u'v')}_{dc}\\
			\end{aligned}
	\end{equation*}\end{small}

	Based on the derived $\frac{\partial \textit{Loss}}{\partial \overline{T}^{(l,uv)}_{dc} } \in\mathbb{C}$ and $\frac{\partial Loss}{\partial \overline{G}^{(l-1,c)}_{u'v'}} \in\mathbb{C}$, we can further compute gradients $\frac{\partial Loss}{\partial (\overline{T}^{(l,uv)})^{\top}}\in\mathbb{C}^{D\times C}$ and $\frac{\partial Loss}{\partial (\overline{\textbf{g}}^{(l-1,u'v')})^{\top}}\in\mathbb{C}^{C}$ as follows.
	\begin{small}\begin{equation}\label{eq:bp_grad_single_T}
			\frac{\partial Loss}{\partial (\overline{T}^{(l,uv)})^{\top}} = \sum_{u'=0}^{M-1} \sum_{v'=0}^{N-1} \chi_{u'v'uv}  \overline{\textbf{g}}^{(l-1,u'v')}  \frac{\partial Loss}{\partial (\overline{\textbf{g}}^{(l,u'v')})^{\top}}
	\end{equation}\end{small}
	\begin{small}\begin{equation}\label{eq:bp_grad_single_g}
			\frac{\partial Loss}{\partial (\overline{\textbf{g}}^{(l-1,u'v')})^{\top}} = \frac{\partial Loss}{\partial (\overline{\textbf{g}}^{(l,u'v')})^{\top}} \overline{T}^{(l,u'v')}
	\end{equation}\end{small}

	Furthermore, \textbf{we extend the above proof of a single convolutional layer to a network with $L$ cascaded convolutional layers.} Let $\textbf{g}^{(l,u'v')}$ denote the frequency component at the frequency $[u',v']$ of the $l$-th layer's output feature, and let $T^{(l,uv)}$ the matrix computed by the $l$-th layer's convolutional weights. Then, according to Equation~(\ref{eq:bp_grad_single_g}), the gradient \emph{w.r.t.} $\overline{\textbf{g}}^{(l,u'v')}$ can be computed as follows.

	\begin{small}\begin{equation}\label{eq:mul_bp}
			\begin{aligned}
				\frac{\partial Loss}{\partial (\overline{\textbf{g}}^{(l,u'v')})^{T}} 
				&= \frac{\partial Loss}{\partial (\overline{\textbf{g}}^{(L,u'v')})^{T}} \overline{T}^{(L,u'v')} \cdots \overline{T}^{(l+1,u'v')}\\
				&= \frac{\partial Loss}{\partial (\overline{\textbf{g}}^{(L,u'v')})^{T}} \overline{\mathbb{T}}^{(u'v')(L:l+1)} \\
			\end{aligned}
	\end{equation}\end{small}

	According to Equation~(\ref{eq:bp_grad_single_T}), the gradient \emph{w.r.t.} $\overline{T}^{(l,uv)}$ can be computed as follows.
	\begin{small}\begin{equation}\label{eq:patial_T}
			\begin{aligned}
				\frac{\partial Loss}{\partial (\overline{T}^{(l,uv)})^{\top}} 
				&= \sum_{u'=0}^{M-1} \sum_{v'=0}^{N-1} \chi_{u'v'uv}  \overline{\textbf{g}}^{(l-1,u'v')}  \frac{\partial Loss}{\partial (\overline{\textbf{g}}^{(l,u'v')})^{\top}}\\
				&= \sum_{u'=0}^{M-1} \sum_{v'=0}^{N-1} \chi_{u'v'uv} (\overline{\mathbb{T}}^{(u'v')(l-1:1)} \overline{\textbf{g}}^{(0,u'v')}+ \overline{\boldsymbol{\beta}}'\delta_{u'v'})\frac{\partial Loss}{\partial (\overline{\textbf{g}}^{(L,u'v')})^{\top}} \overline{\mathbb{T}}^{(u'v')(L:l+1)} {\rm \slash \slash  Corollary~\ref{co1_entire}} \\
				& 
				\quad \slash \slash \text{ Let } \textbf{g}^{(uv)} = \textbf{g}^{(0,uv)};\textbf{h}^{(uv)} = \textbf{g}^{(L,uv)} \\
				&= \sum_{u'=0}^{M-1} \sum_{v'=0}^{N-1} \chi_{u'v'uv} (\overline{\mathbb{T}}^{(u'v')(l-1:1)} \overline{\textbf{g}}^{(u'v')}+ \overline{\boldsymbol{\beta}}'\delta_{u'v'})\frac{\partial Loss}{\partial (\overline{\textbf{h}}^{(u'v')})^{\top}} \overline{\mathbb{T}}^{(u'v')(L:l+1)}
			\end{aligned}
	\end{equation}\end{small}
	
	Let us use the gradient descent algorithm to update the convlutional weight $W^{(l)[\text{ker}=d]}_{c}|_{n}$ of the $n$-th epoch, the updated frequency spectrum $W^{(l)[\text{ker}=d]}_{c}|_{n+1}$ can be computed as follows.
	
	\begin{equation*}
		\forall t,s,\quad W^{(l)[\text{ker}=d]}_{cts}|_{n+1} = W^{(l)[\text{ker}=d]}_{cts}|_{n} - \eta \cdot \frac{\partial Loss}{\partial \overline{W}_{cts}^{(l)[\text{ker=d}]}}
	\end{equation*}
	
	where $\eta$ is the learning rate. Then, the updated frequency spectrum $T^{(l,uv)}|_{n+1}$ computed based on Equation~(\ref{eq:dft_grad}) is given as follows.
	
	\begin{equation*}
		\begin{aligned}
			\Delta T_{dc}^{(l,uv)} 
			&= T_{dc}^{(l,uv)}|_{n+1} - T_{dc}^{(l,uv)}|_{n} \\
			&=  \sum_{t=0}^{K-1}\sum_{s=0}^{K-1} W_{cts}^{(l)[\text{ker}=d]}|_{n+1}e^{i(\frac{ut}{M}+\frac{vs}{N})2\pi} - T_{dc}^{(l,uv)}|_{n}
			\quad {\rm \slash \slash Equation~(\ref{eq:dft})}\\
			&= \sum_{t=0}^{K-1}\sum_{s=0}^{K-1} (W^{(l)[\text{ker}=d]}_{cts}|_{n} - \eta \cdot \frac{\partial Loss}{\partial \overline{W}_{cts}^{(l)[\text{ker=d}]}})  e^{i(\frac{ut}{M}+\frac{vs}{N})2\pi} - T_{dc}^{(l,uv)}|_{n} \\
			&= (\sum_{t=0}^{K-1}\sum_{s=0}^{K-1} W^{(l)[\text{ker}=d]}_{cts}|_{n} e^{i(\frac{ut}{M}+\frac{vs}{N})2\pi}- T_{dc}^{(l,uv)}|_{n}) - \eta \sum_{t=0}^{K-1}\sum_{s=0}^{K-1} \frac{\partial Loss}{\partial \overline{W}_{cts}^{(l)[\text{ker=d}]}}  e^{i(\frac{ut}{M}+\frac{vs}{N})2\pi} \\
			&=  - \eta \sum_{t=0}^{K-1}\sum_{s=0}^{K-1} \frac{\partial Loss}{\partial \overline{W}_{cts}^{(l)[\text{ker=d}]}}  e^{i(\frac{ut}{M}+\frac{vs}{N})2\pi} 
			\quad {\rm \slash \slash Equation~(\ref{eq:dft})}\\
			&= - \eta MN \frac{\partial Loss}{\partial \overline{T}_{dc}^{(l,uv)}}
			\quad {\rm \slash \slash Equation~(\ref{eq:dft_grad})}\\
		\end{aligned}
	\end{equation*}
	
	Therefore, we prove that any step on $W^{(l)[\textit{ker}=d]}_{cts}$  equals to $MN$ step on $T_{dc}^{(uv)}$. In this way, pull Equation~(\ref{eq:patial_T}) in the change of $T^{(l,uv)}$ can be computed as follows.
	\begin{small}
		\begin{equation}
			\left(\Delta T^{(l,uv)}\right)^{\top} = -\eta MN
			\sum_{u'=0}^{M-1}\sum_{v'=0}^{N-1}\chi_{u'v'uv} \left( \overline{\boldsymbol{\mathbb{T}}}^{(u'v')(l-1:1)} \overline{\textbf{g}}^{(u'v')} + \delta_{u'v'}\overline{\boldsymbol{\beta}}' \right) \frac{\partial \textit{Loss}}{\partial (\overline{\textbf{h}}^{(u'v')})^{\top}}  \overline{\boldsymbol{\mathbb{T}}}^{(u'v')(L:l+1)}
		\end{equation}
	\end{small}

\end{proof}

\subsection{Proofs of Assumption~\ref{ass:T} and Theorem~\ref{th:TTT}}\label{app:sec:th4}
We prove Assumption~\ref{ass:T} in the main paper, as follows.

\begin{proof}
	Given an initialized, cascaded, convolutional decoder network with $L$ convolutional layers, let us focus on the behavior of the decoder network in early epochs of training. 
	We notice that each element in the matrix {\small$T^{(l,uv)}$} is exclusively determined by the $c$-th channel of the $d$-th kenel {\small$W^{(l)[\textit{ker}=d]}_{c,1:K,1:K}\in\mathbb{R}^{K\times K}$} according to Theorem~\ref{th2_layerwise}. Because parameters in {\small$W^{(l)}$} in the decoder network are set to random noises, we can consider that all elements in {\small$T^{(l,uv)}$} irrelevant to each other, \emph{i.e.}, {\small$\forall d\ne d',c\ne c',T_{dc}^{(l,uv)}$} is irrelevant to {\small$T_{d'c'}^{(l,uv)}$}. Similarly, since different layers' parameters {\small$W^{(l)}$} are irrelevant to each other in the initialized decoder network, we can consider that elements in different layers' {\small$T^{(l,uv)}$} irrelevant to each other, \emph{i.e.}, $\forall l\ne l'$, elements in {\small$T^{(l,uv)}$} and elements in {\small$T^{(l',uv)}$} are irrelevant to each other. Moreover, since the early training of a DNN mainly modifies a few parameters according to the lottery ticket hypothesis \citep{frankle2018lottery}, we can still assume such irrelevant relationships in early epochs, as follows.
\end{proof}

Then, we prove Theorem~\ref{th:TTT}, as follows.

\begin{proof}
	
	\textbf{We first prove that $T^{(l,uv)}_{dc}$ follows a Gaussian distribution of complex numbers.}
	
	According to Assumption~\ref{ass:T}, each convolutional weight follows a Gaussian distribution, \emph{i.e.}, $W^{\text{ker}=d}_{cts} \sim \mathcal{N}(\mu_{l},\sigma^{2}_{l})$. For the convenience of proving, let us extend $W^{\text{ker}=d}_{cts}$ into an complex number. In this way, $W^{\text{ker}=d}_{cts}$ follows a Gaussian distribution of complex numbers, \emph{i.e.}, $W^{ker=d}_{cts} \sim Complex\mathcal{N}(\mu_{l},\sigma^{2}_{l},0)$.
	
	Previous studies~\cite{tse2005fundamentals} proved that given $N$ complex numbers, if each complex number follows a Gaussian distribution, then the linear summation of these $N$ complex numbers also follows a Gaussian distribution of complex numbers. Since $T^{(l,uv)}_{dc}$ is a linear combination of $\forall t,s,W_{cts}^{(l)[\text{ker}=d]}$, $T^{(l,uv)}_{dc}$ also follows a Gaussian distribution of complex numbers as follows.
	
	\begin{equation*}
		\forall d,c \quad T_{dc}^{(l,uv)} \sim \textit{Complex}\mathcal{N}(\hat{\mu},\hat{\sigma}^2,r)
	\end{equation*}
	
	where
	\begin{equation*}
		\begin{aligned}
			\mu
			&= \mathbb{E}[T^{(l,uv)}_{dc} ] 
			\quad {\rm \slash \slash By\ definetion\ of\ \mu}\\
			&= \mathbb{E}[\sum_{t=0}^{K-1}\sum_{s=0}^{K-1}W_{cts}^{(l)[\text{ker}=d]}e^{i(\frac{ut}{M}+\frac{vs}{N})2\pi}] 
			\quad {\rm \slash \slash Equation~(\ref{eq:dft})}\\
			& {\rm \slash \slash \forall t \ne t' \ or \ s \ne s': \mathbb{E}[W_{cts}^{(l)[\text{ker}=d]}W_{ct's'}^{(l)[\text{ker=d}]}] = \mathbb{E}[W_{cts}^{(l)[\text{ker}=d]}]\mathbb{E}[W_{ct's'}^{(l)[\text{ker=d}]}]} \\
			&=\sum_{t=0}^{K-1}\sum_{s=0}^{K-1}\mathbb{E}[W_{cts}^{(l)[\text{ker}=d]}]e^{i(\frac{ut}{M}+\frac{vs}{N})2\pi} 
			\\
			&=\mu_{l} \sum_{t=0}^{K-1}\sum_{s=0}^{K-1} e^{i(\frac{ut}{M}+\frac{vs}{N})2\pi} 
			\quad {\rm \slash \slash \mathbb{E}[W_{cts}^{(l)[\text{ker}=d]}]= \mu_{l} }\\
			& \quad  \slash \slash Let\ R_{uv} = \sum_{t=0}^{K-1}\sum_{s=0}^{K-1}e^{i(\frac{ut}{M}+\frac{vs}{N})2\pi}\\
			&= \mu_{l}R_{uv}\\
		\end{aligned}
	\end{equation*}
	
	\begin{equation*}
		\begin{aligned}
			\sigma^{2}
			&= \mathbb{E}[(T^{(l,uv)}_{dc}-\mathbb{E}[T^{(l,uv)}_{dc}])\overline{(T^{(l,uv)}_{dc}-\mathbb{E}[T^{(l,uv)}_{dc}])}]
			\quad {\rm \slash \slash By\ definetion\ of\ \sigma^{2}}\\
			&= Var[T^{(l,uv)}_{dc}]\\
			&= Var[\sum_{t=0}^{K-1}\sum_{s=0}^{K-1}W_{cts}^{(l)[\text{ker}=d]}e^{i(\frac{ut}{M}+\frac{vs}{N})2\pi}] 
			\quad {\rm \slash \slash Equation~(\ref{eq:dft})}\\
			& {\rm \slash \slash \forall t \ne t' \ or \ s \ne s': \mathbb{E}[W_{cts}^{(l)[\text{ker}=d]}W_{ct's'}^{(l)[\text{ker=d}]}] = \mathbb{E}[W_{cts}^{(l)[\text{ker}=d]}]\mathbb{E}[W_{ct's'}^{(l)[\text{ker=d}]}]} \\
			&=\sum_{t=0}^{K-1}\sum_{s=0}^{K-1}Var[W_{cts}^{(l)[\text{ker}=d]}e^{i(\frac{ut}{M}+\frac{vs}{N})2\pi}] \\
			&=\sum_{t=0}^{K-1}\sum_{s=0}^{K-1}Var[W_{cts}^{(l)[\text{ker}=d]}]
			\quad  \slash \slash Var[aX] = |a|^{2} Var[X] \\
			&=\sum_{t=0}^{K-1}\sum_{s=0}^{K-1}\sigma_{l}^{2}
			\quad  \slash \slash Var[W_{cts}^{(l)[\text{ker}=d]}] = \sigma_{l}^{2} \\
			&= K^{2} \sigma_{l}^{2} 
		\end{aligned}
	\end{equation*}
	
	\begin{equation*}
		\begin{aligned}
			r
			&= \mathbb{E}[(T^{(l,uv)}_{dc}-\mathbb{E}[T^{(l,uv)}_{dc}])(T^{(l,uv)}_{dc}-\mathbb{E}[T^{(l,uv)}_{dc}])]
			\quad {\rm \slash \slash By\ definetion\ of\ r}\\
			&= C[T^{(l,uv)}_{dc}]
			\quad {\rm \slash \slash Define\ C[\textbf{X}] = \mathbb{E}[(\textbf{X}-\mathbb{E}[\textbf{X}])(\textbf{X}-\mathbb{E}[\textbf{X}])]}\\
			&= C[\sum_{t=0}^{K-1}\sum_{s=0}^{K-1}W_{cts}^{(l)[\text{ker}=d]}e^{i(\frac{ut}{M}+\frac{vs}{N})2\pi}] 
			\quad {\rm \slash \slash Equation~(\ref{eq:dft})}\\
			& {\rm \slash \slash \forall t \ne t' \ or \ s \ne s': \mathbb{E}[W_{cts}^{(l)[\text{ker}=d]}W_{ct's'}^{(l)[\text{ker=d}]}] = \mathbb{E}[W_{cts}^{(l)[\text{ker}=d]}]\mathbb{E}[W_{ct's'}^{(l)[\text{ker=d}]}]}\\
			&=\sum_{t=0}^{K-1}\sum_{s=0}^{K-1}C[W_{cts}^{(l)[\text{ker}=d]}e^{i(\frac{ut}{M}+\frac{vs}{N})2\pi}] \\
			&=\sum_{t=0}^{K-1}\sum_{s=0}^{K-1}C[W_{cts}^{(l)[\text{ker}=d]}] e^{i(\frac{2ut}{M}+\frac{2vs}{N})2\pi}
			\quad  \slash \slash C[aX] = a^{2} C[X] \\
			&=\sigma_{l}^{2} \sum_{t=0}^{K-1}\sum_{s=0}^{K-1} e^{i(\frac{2ut}{M}+\frac{2vs}{N})2\pi}
			\quad  \slash \slash Var[W_{cts}^{(l)[\text{ker}=d]}] = \sigma_{l}^{2} \\
			&= \sigma_{l}^{2} R_{2u,2v} 
			\quad  \slash \slash \ R_{uv} = \sum_{t=0}^{K-1}\sum_{s=0}^{K-1}e^{i(\frac{ut}{M}+\frac{vs}{N})2\pi}\\
		\end{aligned}
	\end{equation*}
	
	Finally, let us consider the value of $R_{uv}$.
	\begin{equation*}
		\begin{aligned}
			R_{uv} 
			&= \sum_{t=0}^{K-1}\sum_{s=0}^{K-1}e^{i(\frac{ut}{M}+\frac{vs}{N})2\pi} \\
			&= \sum_{t=0}^{K-1}e^{i(\frac{2u\pi}{M})t}\sum_{s=0}^{K-1}e^{i(\frac{2v\pi}{N})s} \\
			&= \frac{\sin(\frac{Ku}{M}\pi)}{\sin(\frac{u}{M}\pi)} \cdot \frac{\sin(\frac{Kv}{N}\pi)}{\sin(\frac{v}{N}\pi)} \cdot e^{i(\frac{(K-1)u}{M}+\frac{(K-1)v}{N})\pi}  \quad {\rm \slash \slash According\ to\ Equation~(\ref{eq:sum_of_e})} \\
		\end{aligned}
	\end{equation*}
	
	Therefore, we prove that \textbf{$T^{(l,uv)}_{dc}$ follows a Gaussian distribution of complex numbers.}
	\begin{small}
		\begin{equation}\label{app:eq:T_distribution}
			\forall d,c \quad T_{dc}^{(l,uv)} \sim \textit{Complex}\mathcal{N}(\hat{\mu}=\mu_lR_{uv},\hat{\sigma}^2=K^2\sigma_l^2,r=\sigma_l^2R_{2u,2v})
		\end{equation}
		\begin{equation*}
			\text{s.t.}\ \ R_{uv} = 
			\frac{\sin(uK\pi/M)}{\sin(u\pi/M)} \frac{\sin(vK\pi/N)}{\sin(v\pi/N)} e^{i(\frac{(K-1)u}{M}+\frac{(K-1)v}{N})\pi}
		\end{equation*}
	\end{small}
	
	\textbf{Then, we prove Theorem~\ref{th:TTT} as follows.}

	According to Equation (\ref{app:eq:T_distribution}), $\forall d,c,l: \mathbb{E}[T^{(l,uv)}_{dc}] = \mu_{l}R_{uv}, Var[T^{(l,uv)}_{dc}] = K^{2}\sigma^{2}_{l}$.
	
	\begin{equation}\label{eq:som}
		\begin{aligned}
			SOM(T^{(l,uv)}_{dc}) 
			& = \mathbb{E}[|T^{(l,uv)}_{dc}|^{2}]\\
			& = |\mathbb{E}[T^{(l,uv)}_{dc}]|^{2} + Var[T^{(l,uv)}_{dc}] \\
			& = |\mu_{l}R_{uv}|^{2} + K^{2}\sigma^{2}_{l}
		\end{aligned}
	\end{equation}
	
	Then, we have
	\begin{equation*}
		\begin{aligned}
			\log(SOM(\mathbb{T}^{(uv)(L:1)}))
			& = \log(\mathbb{E}[|\mathbb{T}^{(uv)(L:1)}|^{2}])\\
			& = \log(\mathbb{E}[|T^{(L,uv)}\mathbb{T}^{(uv)(L-1:1)}|^{2}])\\
			& {\rm \slash \slash According\ to \ Assumption~\ref{ass:T},\ and\ C_{l} = 1} \\
			& = \log(\mathbb{E}[|T^{(L,uv)}|^{2}] \mathbb{E}[|\mathbb{T}^{(uv)(L-1:1)}|^{2}] )\\
			& = \log((|\mu_{L}R_{uv}|^{2} + K^{2}\sigma^{2}_{L}) SOM(\mathbb{T}^{(uv)(L-1:1)})  )
			\quad {\rm \slash \slash Equation~(\ref{eq:som})}\\
			& = \log(\prod_{l=1}^{L}|\mu_{l}R_{uv}|^{2} + K^{2}\sigma^{2}_{l}) \\ 
			& = \sum_{l=1}^{L} \log(|\mu_{l}R_{uv}|^{2} + K^{2}\sigma^{2}_{l}) 
		\end{aligned}
	\end{equation*}
	
\end{proof}

\textbf{For the more general case that each convolutional kernel contains more than one channel, \emph{i.e.}, {\small$\forall l,C_l>1$}, the {\small$\textit{SOM}(\boldsymbol{\mathbb{T}}^{(uv)(L:1)})$} also approximately exponentially increases along with the depth of the network with a quite complicated analytic solution, as proved below.} Note that the following proof is based Assumption~\ref{ass:T}. Besides, we further assume that all elements in $\mathbb{T}^{(uv)(l:1)}$ are independent with each other. \emph{I.e.}, $\forall d \ne d';c \ne c', \mathbb{E}[\mathbb{T}_{dc}^{(uv)(l:1)}\mathbb{T}_{d'c'}^{(uv)(l:1)}] = \mathbb{E}[\mathbb{T}_{dc}^{(uv)(l:1)}]\mathbb{E}[\mathbb{T}_{d'c'}^{(uv)(l:1)}] $.

\begin{proof}
	According to Equation~(\ref{app:eq:T_distribution}), all elements in ${T}^{(l,uv)}$ follow the same Gaussian distribution. Therefore, we have
	\begin{equation}\label{eq:T_mean}
		\begin{aligned}
			\mathbb{E}[T^{(l,uv)}] 
			&= \mathbb{E}[T^{(l,uv)}_{dc}] \textbf{1}_{(C_{l}\times C_{l-1})}\\
			&= \mu_{l}R_{uv} \textbf{1}_{(C_{l}\times C_{l-1})} \\
		\end{aligned}
	\end{equation}
	
	and we have 
	\begin{equation}\label{eq:T_som}
		\begin{aligned}
			SOM({T}^{(l,uv)}) 
			&= SOM({T}^{(l,uv)}_{dc}) \textbf{1}_{(C_{l}\times C_{l-1})}\\
			&= (|\mu_{l}R_{uv}|^{2} + K^{2}\sigma^{2}_{l})\textbf{1}_{(C_{l}\times C_{l-1})} 
		\end{aligned}
	\end{equation}
	
	Let us first consider the expectation of $\mathbb{T}^{(uv)(L:1)}$ as follows.
	
	\begin{equation}\label{eq:bbT_mean}
		\begin{aligned}
			\mathbb{E}[\mathbb{T}^{(uv)(L:1)}] 
			&= \mathbb{E}[T^{(L,uv)}\mathbb{T}^{(uv)(L-1:1)}] \\
			&= (C_{L-1}\mathbb{E}[T_{dc}^{(L,uv)}]\mathbb{E}[\mathbb{T}_{dc}^{(uv)(L-1:1)}])\textbf{1}_{(C_{L}\times C_{0})} 
			\quad {\rm \slash \slash Assumption~\ref{ass:T}, Equation~(\ref{eq:T_mean}})\\
			&= (C_{L-1} \mu_{l} R_{uv} \mathbb{E}[\mathbb{T}_{dc}^{(uv)(L-1:1)}])\textbf{1}_{(C_{L}\times C_{0})} 
			\quad {\rm \slash \slash Equation\ (\ref{app:eq:T_distribution})} \\
			&= \left(\frac{1}{C_{L}} \prod_{l=1}^{L} C_{l}\mu_{l} R_{uv}\right)\textbf{1}_{(C_{L}\times C_{0})} 
			\quad {\rm \slash \slash Assumption~\ref{ass:T}} \\
		\end{aligned}
	\end{equation}
	
	Then, we have
	\begin{small}
		\begin{equation}\label{eq:bbT_som}
			\begin{aligned}
				&SOM(\mathbb{T}^{(uv)(L:1)}) \\
				&= \mathbb{E}[|\mathbb{T}^{(uv)(L:1)}|^{2}]\\
				&= \mathbb{E}[|T^{(L,uv)}\mathbb{T}^{(uv)(L-1:1)}|^{2}]\\
				&= (C_{L-1} SOM(T_{dc}^{(L,uv)})SOM(\mathbb{T}^{(uv)(L-1:1)}_{dc}) + C_{L-1}(C_{L-1}-1)|\mathbb{E}[T_{dc}^{(L,uv)}]\mathbb{E}[\mathbb{T}_{dc}^{(uv)(L-1:1)}]|^{2})\textbf{1}_{(C_{L}\times C_{0})} \\
				& {\rm \slash \slash According\ to\ Assumption~\ref{ass:T}\ and\ Equation~(\ref{eq:T_som}), } \\
				& {\rm \slash \slash  we\ further\ Assume\ \forall d \ne d';c \ne c', \mathbb{E}[\mathbb{T}_{dc}^{(uv)(l:1)}\mathbb{T}_{d'c'}^{(uv)(l:1)}] = \mathbb{E}[\mathbb{T}_{dc}^{(uv)(l:1)}]\mathbb{E}[\mathbb{T}_{d'c'}^{(uv)(l:1)}] }\\
				&= (C_{L-1} (|\mu_{L}R_{uv}|^{2} + K^{2}\sigma^{2}_{L}) SOM(\mathbb{T}^{(uv)(L-1:1)}_{dc}) + \frac{C_{L-1}-1}{C_{L-1}}|\mathbb{E}[\mathbb{T}_{dc}^{(uv)(L:1)}]|^{2})\textbf{1}_{(C_{L}\times C_{0})} \\
				& {\rm \slash \slash According\ to\  Equation~(\ref{eq:som}),\ Equation~(\ref{eq:bbT_mean})} \\
				&= \bigg(\frac{1}{C_{L}}\prod_{l=1}^{L}C_{l}(|\mu_{l} R_{u,v}|^{2} + (K\sigma_{l})^{2}) + \sum_{l=2}^{L} \frac{C_{l-1}-1}{C_{l-1}} |\frac{1}{C_{l}}\prod_{k=1}^{l}C_{k}\mu_{k} R_{u,v}|^{2} \prod_{j=l+1}^{L}C_{j-1} \Big(|\mu_{j} R_{u,v}|^{2}  + (K\sigma_{j})^{2}\Big) \bigg) \textbf{1}_{C_{L} \times C_{0}}
			\end{aligned}
		\end{equation}
	\end{small}
	
	Therefore, we prove that for the more general case that {\small$\forall l,C_l>1$}, the second-order moment {\small$\textit{SOM}(\boldsymbol{\mathbb{T}}^{(uv)(L:1)})$} also approximately exponentially increases along with the depth of the network.
	
\end{proof}

\subsection{Proof of Theorem~\ref{th:zero_padding}}\label{app:sec:th5}
In this section, we prove Theorem \ref{th:zero_padding} in the main paper, as follows.

\begin{proof}
	
	\begin{equation}\label{EG}
		\begin{aligned}
			\mathbb{E}_{F^{(c)}}[G_{uv}^{(c)}]
			&=\mathbb{E}[ \sum_{m=0}^{M-1}\sum_{n=0}^{N-1} F^{(c)}_{mn} e^{-i(\frac{um}{M}+\frac{vn}{N})2\pi} ]
			\quad {\rm \slash \slash Equation~(\ref{eq:dft})} \\
			&=\sum_{m=0}^{M-1}\sum_{n=0}^{N-1} \mathbb{E}[ F^{(c)}_{mn} ]e^{-i(\frac{um}{M}+\frac{vn}{N})2\pi}  \\ 
			&= a \sum_{m=0}^{M-1}\sum_{n=0}^{N-1}e^{-i(\frac{um}{M}+\frac{vn}{N})2\pi} 
			\quad {\rm \slash \slash F^{(c)}_{mn} \sim \mathcal{N}(a,\sigma^{2})} \\
			& = a MN \delta_{uv};0\le u<M,0\le  v<N
			\quad {\rm \slash \slash Equation~(\ref{eq:delta_uv})} \\
		\end{aligned}
	\end{equation}
	
	\begin{equation}\label{EH}
		\begin{aligned}
			&\mathbb{E}_{F^{(c)}}[H_{uv}^{(c)}] \\
			&=\mathbb{E}_{F^{(c)}}[ \sum_{m=0}^{M'-1}\sum_{n=0}^{N'-1} \tilde{F}^{(c)}_{mn} e^{-i(\frac{um}{M'}+\frac{vn}{N'})2\pi} ]
			\quad {\rm \slash \slash Equation~(\ref{eq:dft})} \\
			&= \mathbb{E}_{F^{(c)}}[\sum_{m=0}^{M-1}\sum_{n=0}^{N-1} F^{(c)}_{mn} e^{-i(\frac{um}{M'}+\frac{vn}{N'})2\pi} ] \\
			&= a \sum_{m=0}^{M-1}\sum_{n=0}^{N-1}e^{-i(\frac{um}{M'}+\frac{vn}{N'})2\pi} 
			\quad {\rm \slash \slash F^{(c)}_{mn} \sim \mathcal{N}(a,\sigma^{2})} \\
			&= a\frac{\sin(\frac{Mu}{M'}\pi)}{\sin(\frac{u}{M'}\pi)}\frac{\sin(\frac{Nv}{N'}\pi)}{\sin(\frac{v}{N'}\pi)}e^{-i(\frac{(M-1)u}{M'}+\frac{(N-1)v}{N'})\pi};0\le u<M',0\le v<N'
			\quad {\rm \slash \slash Equation~(\ref{eq:sum_of_e})} \\
		\end{aligned}
	\end{equation}

	\begin{small}\begin{equation}\label{VarG}
			\begin{aligned}
				\textit{Var}_{F^{(c)}}[G_{uv}^{(c)}] 
				&= \textit{Var}[\sum_{m=0}^{M-1}\sum_{n=0}^{N-1} F_{mn}^{(c)}e^{-i(\frac{um}{M}+\frac{vn}{N})2\pi}]\\
				&=\sum_{m=0}^{M-1}\sum_{n=0}^{N-1} \textit{Var}[F_{mn}^{(c)}e^{-i(\frac{um}{M}+\frac{vn}{N})2\pi}]
				\quad \slash\slash \forall m,n;
				F_{mn}^{(c)} \text{\ is i.i.d}
				\\
				&= \sum_{m=0}^{M-1}\sum_{n=0}^{N-1} \textit{Var}[F_{mn}^{(c)}]\\ 
				&= \sum_{m=0}^{M-1}\sum_{n=0}^{N-1} \sigma^{2} \quad \slash\slash F^{(c)}_{mn}\sim \mathcal{N}(a,\sigma^{2})\\
				&= MN \sigma^{2}
			\end{aligned}
		\end{equation}
	\end{small}
	
	\begin{small}\begin{equation}\label{VarH}
			\begin{aligned}
				\textit{Var}_{F^{(c)}}[H_{uv}^{(c)}] 
				&= \textit{Var}[\sum_{m=0}^{M'-1}\sum_{n=0}^{N'-1} \tilde{F}_{mn}^{(c)}e^{-i(\frac{um}{M'}+\frac{vn}{N'})2\pi}]\\
				&= \textit{Var}[\sum_{m=0}^{M-1}\sum_{n=0}^{N-1} F_{mn}^{(c)}e^{-i(\frac{um}{M'}+\frac{vn}{N'})2\pi}]\\
				&=\sum_{m=0}^{M-1}\sum_{n=0}^{N-1} \textit{Var}[F_{mn}^{(c)}e^{-i(\frac{um}{M'}+\frac{vn}{N'})2\pi}]
				\quad \slash\slash \forall m,n;
				F_{mn}^{(c)} \text{\ is i.i.d}
				\\
				&= \sum_{m=0}^{M-1}\sum_{n=0}^{N-1} \textit{Var}[F_{mn}^{(c)}]\\
				&= \sum_{m=0}^{M-1}\sum_{n=0}^{N-1} \sigma^{2}
				\quad \slash\slash F^{(c)}_{mn}\sim \mathcal{N}(a,\sigma^{2})\\
				&= MN \sigma^{2};0\le u<M',0\le v<N'
			\end{aligned}
		\end{equation}
	\end{small}
	
	When $0\le u<M,0\le v<N$:
	
	\begin{small}\begin{equation}
			\begin{aligned}
				\textit{SOM}(H_{uv}^{(c)}) - \textit{SOM}(G_{uv}^{(c)})
				&= |\mathbb{E}[H_{uv}^{(c)}]|^{2}  + \textit{Var}(H_{uv}^{(c)}) 
				- (|\mathbb{E}[G_{uv}^{(c)}]|^{2}  + \textit{Var}(G_{uv}^{(c)}))\\
				&= (a\tau_{uv})^{2}  + MN\sigma^{2} - (aMN)^{2}\delta_{uv} - MN\sigma^{2} \\
				&// \text{According to Equation~(\ref{EG}),Equation~(\ref{EH}),Equation~(\ref{VarG})  and Equation~(\ref{VarH})}\\
				&= (a\tau_{uv})^{2} - (aMN)^{2}\delta_{uv} \\
			\end{aligned}
		\end{equation}
	\end{small}
	
	Therefore, We prove that $\forall 0\le u<M,0\le v<N,u+v\ne 0$
	\begin{small}\begin{equation}
			\begin{aligned}
				\textit{SOM}(H_{uv}^{(c)}) - \textit{SOM}(G_{uv}^{(c)})
				&= (a\tau_{uv})^{2} \\
			\end{aligned}
		\end{equation}
	\end{small}
	
	When $u=v=0$:
	\begin{small}\begin{equation}
			\begin{aligned}
				\textit{SOM}(H_{uv}^{(c)}) &= \textit{SOM}(G_{uv}^{(c)})\\
			\end{aligned}
		\end{equation}
	\end{small}

\end{proof}

\subsection{Proof of Theorem~\ref{th_upconv}}\label{app:sec:th6}
In this section, we prove Theorem~\ref{th_upconv} in the main paper, as follows.

\begin{proof}
	\begin{equation}\label{eq:up_in}
		G_{uv}^{(c)} =\sum_{m=0}^{M_{0}-1}\sum_{n=0}^{N_{0}-1} F_{mn}^{(c)} e^{-i(\frac{um}{M_{0}}+\frac{vn}{N_{0}})2\pi}
		\quad {\rm \slash \slash  Equation~(\ref{eq:dft})} \\
	\end{equation}
	
	\begin{small}\begin{equation}\label{eq:up_out}
			\begin{aligned}
				H_{u+(s-1)M_{0},v+(t-1)N_{0}}^{(c)}
				&= \sum_{m=0}^{M-1}\sum_{n=0}^{N-1} \tilde{F}_{mn}^{(c)} e^{-i(\frac{(u+(s-1)M_{0})m}{M}+\frac{(v+(t-1)N_{0})n}{N})2\pi}
				\quad {\rm \slash \slash  Equation~(\ref{eq:dft})} \\
				&= \sum_{m=0}^{M_{0} -1}\sum_{n=0}^{N_{0}-1} F_{mn}^{(c)} e^{-i(\frac{(u+(s-1)M_{0})(m\cdot ratio)}{M}+\frac{(v+(t-1)N_{0})(n \cdot ratio)}{N})2\pi} \\
				&= \sum_{m=0}^{M_{0} -1}\sum_{n=0}^{N_{0}-1} F_{mn}^{(c)} e^{-i(\frac{(u+(s-1)M_{0})m}{M/ratio}+\frac{(v+(t-1)N_{0})n}{N/ratio})2\pi}\\
				& \slash \slash  M=M_{0} \cdot ratio;N= N_{0}\cdot ratio \\
				&= \sum_{m=0}^{M_{0} -1}\sum_{n=0}^{N_{0}-1} F_{mn}^{(c)} e^{-i(\frac{(u+(s-1)M_{0})m}{M_{0}}+\frac{(v+(t-1)N_{0})n}{N_{0}})2\pi}
				\\
				&= \sum_{m=0}^{M_{0} -1}\sum_{n=0}^{N_{0}-1} F_{mn}^{(c)} e^{-i(\frac{um}{M_{0}}+\frac{vn}{N_{0}})2\pi}\cdot e^{-i((s-1)m+(t-1)n)2\pi}\\
				&= \sum_{m=0}^{M_{0} -1}\sum_{n=0}^{N_{0}-1} F_{mn}^{(c)} e^{-i(\frac{um}{M_{0}}+\frac{vn}{N_{0}})2\pi}
				\quad {\rm \slash \slash  s,t \in \mathcal{Z}} \\
				&= G_{uv}^{(c)} 
				\quad {\rm \slash \slash Equation~(\ref{eq:up_in})} \\
			\end{aligned}
	\end{equation}\end{small}
	
	Therefore we prove that:
	\begin{small}
		\begin{equation}
			\forall c,u,v,\quad
			H_{u+(s-1)M_0,v+(t-1)N_0}^{(c)}=G_{uv}^{(c)} \qquad \quad \text{s.t.}\ \ s=1,\ldots,M/M_0;\ t=1,\ldots,N/N_0	
		\end{equation}
	\end{small}
\end{proof}

\subsection{Proof of Theorem~\ref{th:cost}}\label{app:sec:difficulty}

In this section, we prove Theorem \ref{th:cost} in the main paper. Recall that we consider the input {\small$x\in\mathbb{R}^{M\times N}$} with a single channel to simplify the proof. 
Then, {\small$G\in \mathbb{C}^{M\times N}$}, {\small$H\in \mathbb{C}^{M\times N}$}, and {\small$H^*\in \mathbb{C}^{M\times N}$} denote spectrums of the input, the output, and the target image to fit, respectively. Specifically, {\small$H^*_{u_1v_1} = (1-A)G^{(u_1v_1)}$}, and {\small$H^*_{u_2v_2} = (1+\overline{A})G^{(u_2v_2)}$}, where {\small$A=\alpha e^{i\phi}$} and {\small$\overline{A}$} denotes the conjugate of {\small$A$}; {\small$\alpha>0$}; {\small$\phi<\pi/2$}.
Theorem~\ref{th:cost} proves a case that weights {\small$\textbf{W}$} is optimized to satisfy {\small$H_{u_1v_1}=H^*_{u_1v_1}$} and {\small$H_{u_2v_2}=H^*_{u_2v_2}$} simultaneously.

\begin{proof}
	
	Let us provide specific constrains as follows, so as to make the auto-encoder learnable (i.e., ensuring {\small$\Delta\textbf{W}$} is real-valued) and make the objective function can reach zero by a single step of gradient descent.
	\begin{equation} \label{a1}
		\lambda_1 = \frac{1}{|G_{u_1v_1}|^2}; \quad\quad \lambda_2 = \frac{1}{|G_{u_2v_2}|^2}
	\end{equation}
	\begin{equation}\label{a2}
		A=\alpha e^{i\phi}, \quad \textit{s.t.} \ \phi = (\frac{(K-1)(u_2-u_1)}{M} + \frac{(K-1)(v_2-v_1)}{N})\frac{\pi}{2}
	\end{equation}
	\begin{equation} \label{a3}
		\mathbb{T}^{(u_2v_2)(L:1)} = \mathbb{T}^{(u_1v_1)(L:1)}  = 1
	\end{equation}
	\begin{equation} \label{a4}
		\sum_{l=1}^{L}\left \| \mathbb{T}^{(u_1v_1)(L:l+1)} \right \|^{2}\left \| \mathbb{T}^{(u_1v_1)(l-1:1)} \right \|^{2} = 
		\sum_{l=1}^{L}\left \| \mathbb{T}^{(u_2v_2)(L:l+1)} \right \|^{2}\left \| \mathbb{T}^{(u_2v_2)(l-1:1)} \right \|^{2}
	\end{equation}

	If the objective function can reach zero (\emph{i.e.}, {\small$H_{u_1v_1}=H^*_{u_1v_1}$} and {\small$H_{u_2v_2}=H^*_{u_2v_2}$}) by a single step of gradient descent, the change of weights $\Delta \textbf{W}$ can be rewritten as $\eta \frac{\textit{Loss}}{\partial \textbf{W}}$, where $\eta$ denotes the learning rate, $\frac{\textit{Loss}}{\partial \textbf{W}}$ denotes the gradient on weights. Then, we will prove the formulations of $\eta$ and $\frac{\textit{Loss}}{\partial \textbf{W}}$.
	
	According to Corollary~\ref{co1_entire}, the network output $H_{uv}$ can be computed as follows.
	\begin{equation}
		H_{uv} = \mathbb{T}^{(uv)(L:1)} G_{uv},
	\end{equation}
	
	where $\mathbb{T}^{(uv)(L:1)} = T^{(L,uv)}T^{(L-1,uv)}\cdots T^{(1,uv)} \in \mathbb{C}^{1\times 1}$.
	
	According to Corollary~\ref{co2_back}, after a single step of gradient descent, the change of $\mathbb{T}^{(uv)(L:1)}$ can be computed as follows.
	\begin{small}
		\begin{equation}
			\begin{aligned}
				\Delta &\mathbb{T}^{(uv)(L:1)} \\
				& =  \Delta (T^{(L,uv)}T^{(L-1,uv)}\cdots T^{(1,uv)}) \\
				&\approx  \sum_{l=1}^{L} \mathbb{T}^{(L:l+1)(uv)} \Delta T^{(l,uv)} \mathbb{T}^{(l-1:1)(uv)} 
				\quad {\rm \slash \slash \text{First order approximation}}\\
				&=-\eta MN \sum_{l=1}^{L} \mathbb{T}^{(L:l+1)(uv)}( \sum_{u'v'}\chi_{u'v'uv}  \overline{\mathbb{T}}^{(l-1:1)(uv)} \overline{G}_{uv} \frac{\partial Loss}{\partial \overline{H}_{uv}^{T}} \mathbb{T}^{(L:l-1)(uv)})^T \mathbb{T}^{(l-1:1)(uv)}\\
				&= -2\eta MN (\lambda_1\chi_{u_1v_1uv} \overline{G}_{u_1v_1}(H_{u_1v_1}-H^*_{u_1v_1}) + \lambda_2\chi_{u_2v_2uv} \overline{G}_{u_2v_2}(H_{u_2v_2}-H^*_{u_2v_2})) \sum_{l=1}^{L}\left \|\mathbb{T}^{(L:l+1)(uv)}\right \|^{2}  \left \|\mathbb{T}^{(l-1:1)(uv)}\right \|^{2} \\
				& 
				{\rm \slash \slash According\ to\ Equation~(\ref{a1})}\\
				&= -2\eta MN (\chi_{u_1v_1uv} (\mathbb{T}^{(L:1)(u_1v_1)}-\frac{H^*_{u_1v_1}}{G_{u_1v_1}}) + \chi_{u_2v_2uv} (\mathbb{T}^{(L:1)(u_2v_2)}-\frac{H^*_{u_2v_2}}{G_{u_2v_2}})) \sum_{l=1}^{L}\left \|\mathbb{T}^{(L:l+1)(uv)}\right \|^{2}  \left \|\mathbb{T}^{(l-1:1)(uv)}\right \|^{2} \\
				&= -2\eta MN  (A\chi_{u_1v_1uv}-\overline{A}\chi_{u_2v_2uv} ) \sum_{l=1}^{L}\left \|\mathbb{T}^{(L:l+1)(uv)}\right \|^{2}  \left \|\mathbb{T}^{(l-1:1)(uv)}\right \|^{2} {\rm \slash \slash Equation~(\ref{a3})}\\
			\end{aligned}
		\end{equation}
	\end{small}
	
	For frequencies $[u_1,v_1]$ and $[u_2,v_2]$, we have:
	
	\begin{small}\begin{equation}
			\Delta \mathbb{T}^{(L:1)(u_1v_1)}= -2\eta MN  (A\chi_{u_1v_1u_1v_1}-\overline{A}\chi_{u_2v_2u_1v_1} ) \sum_{l=1}^{L}\left \|\mathbb{T}^{(L:l+1)(u_1v_1)}\right \|^{2}  \left \|\mathbb{T}^{(l-1:1)(u_1v_1)}\right \|^{2} 
	\end{equation}\end{small}
	\begin{small}\begin{equation}
			\Delta \mathbb{T}^{(L:1)(u_2v_2)} = -2\eta MN (A\chi_{u_1v_1u_2v_2}-\overline{A}\chi_{u_2v_2u_2v_2} ) \sum_{l=1}^{L}\left \|\mathbb{T}^{(L:l+1)(u_2v_2)}\right \|^{2}  \left \|\mathbb{T}^{(l-1:1)(u_2v_2)}\right \|^{2}
	\end{equation}\end{small}
	
	For any other frequency component $[u,v]$, where $u\ne u_1, u\ne u_2, v\ne v_1, v\ne v_2$, the change $\Delta \mathbb{T}^{(uv)(L:1)}$ can be cpmputed as the linear combination of the $\Delta \mathbb{T}^{(L:1)(u_1v_1)}$ and $\Delta \mathbb{T}^{(L:1)(u_2v_2)}$ as follows.
	
	\begin{small}\begin{equation}\label{d1}
			\Delta \mathbb{T}^{(uv)(L:1)} = a_1 \Delta \mathbb{T}^{(L:1)(u_1v_1)} + a_2 \Delta \mathbb{T}^{(L:1)(u_2v_2)}
	\end{equation}\end{small}
	where $a_1\in\mathbb{C}$ and $a_2\in\mathbb{C}$ are two complex coefficients, which keep unchanged during the learning process.
	
	On the other hand, the exact change of $\mathbb{T}^{(uv)(L:1)}$ can be directly computed given the objective function. For frequencies $[u_1,v_1]$ and $[u_2,v_2]$, we have:
	
	\begin{equation}
		\Delta \mathbb{T}^{(L:1)(u_1v_1)} = \frac{H^*_{u_1v_1}}{G_{u_1v_1}} - \mathbb{T}^{(L:1)(u_1v_1)}=-A    
	\end{equation}
	\begin{equation}
		\Delta \mathbb{T}^{(L:1)(u_2v_2)} = \frac{H^*_{u_2v_2}}{G_{u_2v_2}} - \mathbb{T}^{(L:1)(u_2v_2)}=\overline{A}    
	\end{equation}
	
	For any other frequency component $[u,v]$, the change $\Delta \mathbb{T}^{(uv)(L:1)}$ can be computed as follows:
	
	\begin{small}\begin{equation}\label{d2}
			\Delta \mathbb{T}^{(uv)(L:1)} = -a_1 A + a_2 \overline{A}
	\end{equation}\end{small}
	
	Then, combining Equation~(\ref{d1}) and Equation~(\ref{d2}), we can obtain the value of $\eta$.
	\begin{small}\begin{equation}\label{eta}
			\begin{aligned}
				\eta 
				&\propto \frac{-a_1 A + a_2 \overline{A}}{-a_1 (A\chi_{u_1v_1u_1v_1}-\overline{A}\chi_{u_2v_2u_1v_1} )  - a_2 (A\chi_{u_1v_1u_2v_2}-\overline{A}\chi_{u_2v_2u_2v_2} ) } \quad {\rm \slash \slash Equation~(\ref{a4})} \\
				&=  \frac{-a_1 A + a_2 \overline{A}}{-a_1A (\chi_{u_1v_1u_1v_1}-e^{-i2\phi }\chi_{u_2v_2u_1v_1} )  + a_2\overline{A}  (\chi_{u_2v_2u_2v_2} - e^{i2\phi }\chi_{u_1v_1u_2v_2}) } \\
				&= \frac{-a_1 A + a_2 \overline{A}}{-a_1A   + a_2\overline{A}   } \cdot \frac{MN}{K^{2} - \frac{\sin(K(u_2-u_1)\pi/M)\sin(K(v_2-v_1)\pi/N)}{\sin((u_2-u_1)\pi/M)\sin((v_2-v_1)\pi/N)}} \\
				&= \frac{MN}{K^{2} - \frac{\sin(K(u_2-u_1)\pi/M)\sin(K(v_2-v_1)\pi/N)}{\sin((u_2-u_1)\pi/M)\sin((v_2-v_1)\pi/N)}}
			\end{aligned}
	\end{equation}\end{small}
	
	And $\left \|\frac{\textit{Loss}}{\partial \textbf{W}}\right \|$ can be computed as follows. 
	
	\begin{small}\begin{equation}\label{partial}
			\begin{aligned}
				\left \|\frac{\textit{Loss}}{\partial \textbf{W}}\right \|^{2} 
				&=
				\sum_{l,d,c}\sum_{t,s} (\frac{\partial \textit{Loss}}{\partial \overline{\textbf{W}}_{cts}^{(l)[\text{ker=d}]}})^{2} \\
				&=\sum_{l,d,c} \sum_{t,s} \frac{\partial \textit{Loss}}{\partial \overline{\textbf{W}}_{cts}^{(l)[\text{ker=d}]}} \cdot \frac{\partial \textit{Loss}}{\partial \overline{\textbf{W}}_{cts}^{(l)[\text{ker=d}]}} \\
				&= \sum_{l,d,c} \sum_{t,s} \frac{\partial \textit{Loss}}{\partial \overline{\textbf{W}}_{cts}^{(l)[\text{ker=d}]}} 
				\sum_{u,v}\frac{\partial \textit{Loss}}{\partial \overline{T}^{(l,uv)}_{dc}}e^{-i(\frac{ut}{M}+\frac{vs}{N})2\pi} \\
				&= \sum_{l,d,c} \sum_{u,v}\frac{\partial \textit{Loss}}{\partial \overline{T}^{(l,uv)}_{dc}}  \sum_{t,s} \frac{\partial \textit{Loss}}{\partial \overline{\textbf{W}}_{cts}^{(l)[\text{ker=d}]}}  e^{-i(\frac{ut}{M}+\frac{vs}{N})2\pi} \\
				&= \sum_{l,d,c} \sum_{u,v}\frac{\partial \textit{Loss}}{\partial \overline{T}^{(l,uv)}_{dc}} \frac{\partial \textit{Loss}}{\partial T^{(l,uv)}_{dc}} \\ 
				&=\sum_{l,d,c}  \sum_{u,v}\frac{\partial \textit{Loss}}{\partial \overline{T}^{(l,uv)}_{dc}} \frac{\partial \textit{Loss}}{\partial T^{(l,uv)}_{dc}} \\
				&=\sum_{l,d,c} \sum_{u,v}|\frac{\partial \textit{Loss}}{\partial \overline{T}^{(l,uv)}_{dc}}|^{2} \\
				&\propto \sum_{l,d,c} \sum_{u,v} \alpha^{2} \\
				&\propto \alpha^{2}
			\end{aligned}
	\end{equation}\end{small}
	
	Therefore, the weight change $\left \|  \Delta \textbf{W} \right \|$ can be computed as follows:
	
	\begin{small}\begin{equation}
			\begin{aligned}
				\left \|  \Delta W \right \| 
				&= \eta \cdot \left \| \frac{Loss}{\partial W} \right \| \\
				& \propto \eta \cdot \sqrt{\left \|  \Delta W \right \|^{2}} \\
				&\propto \frac{MN \alpha}{K^{2} - \frac{\sin(K(u_2-u_1)\pi/M)\sin(K(v_2-v_1)\pi/N)}{\sin((u_2-u_1)\pi/M)\sin((v_2-v_1)\pi/N)}}  \quad {\rm \slash \slash Equation~(\ref{eta})\ and\ Equation~(\ref{partial})}
			\end{aligned}
		\end{equation}
	\end{small}
	
\end{proof}

\section{Discussions about different factors that weakening high-frequency components}

\subsection{Effects of the network depth}\label{app_discussion_L}

If the decoder network is deep, then the decoder network is less likely to learn high-frequency components. It is because {\small$\lvert R_{uv}\rvert$} is relatively large for low-frequency components. In this way, the large effect of a single layer's {\small$T^{(l,uv)}$} of low-frequency components on {\small$\log\textit{SOM}(\boldsymbol{\mathbb{T}}^{(uv)(L:1)})$}, \emph{i.e.}, {\small$\log(\lvert\mu_lR_{uv}\rvert^2+K^2\sigma_l^2)$}, can be accumulated through different layers according to (1) the Law of Large Numbers, and (2) the independent effects {\small$\log(\lvert\mu_lR_{uv}\rvert^2+K^2\sigma_l^2)$} between different layers' {\small$T^{(l,uv)}$} on {\small$\log\textit{SOM}(\boldsymbol{\mathbb{T}}^{(uv)(L:1)})=\sum\nolimits_{l=1}^L \log(\lvert\mu_lR_{uv}\rvert^2+K^2\sigma_l^2)$}.

Therefore, the large {\small$\lvert R_{u^{\text{low}}v^{\text{low}}}\rvert$} value for a low-frequency component {\small$[u^{\text{low}},v^{\text{low}}]$} makes {\small$\boldsymbol{\mathbb{T}}^{(u^{\text{low}}v^{\text{low}})(L:1)}$} more likely to have a large norm, whereas the small {\small$\lvert R_{u^{\text{high}}v^{\text{high}}}\rvert$} value for a high-frequency component {\small$[u^{\text{high}},v^{\text{high}}]$} makes {\small$\boldsymbol{\mathbb{T}}^{(u^{\text{high}}v^{\text{high}})(L:1)}$} less likely to have a large norm. This indicates that \textbf{a deep decoder network will almost certainly strengthen the encoding of low-frequency components of the input sample, while weaken the encoding of high-frequency components.}

\subsection{Effects of the initialization of network parameters}\label{app_discussion_mu}

If the expectation {\small$\mu_l$} of convolutional weights in each $l$-th layer has a large absolute value {\small$\lvert\mu_l\rvert$}, then the decoder network is less likely to learn high-frequency components. It is because according to Theorem~\ref{th:TTT}, a large absolute value {\small$\lvert\mu_l\rvert$} boosts the imbalance effects {\small$\lvert\mu_lR_{uv}\rvert^2$} among different frequency components, thereby strengthening the trend of encoding low-frequency components of the input sample.

\subsection{Effects of the convolutional kernel size}\label{app_discussion_K}

If the convolutional kernel size $K$ is small, then the decoder network is less likely to learn high-frequency components. It is because according to Theorem \ref{th:TTT}, a large $K$ value alleviates imbalance of the second-order moment {\small$\textit{SOM}(\boldsymbol{\mathbb{T}}^{(uv)(L:1)})$} between low frequencies and high frequencies caused by the imbalance of {\small$\lvert R_{uv}\rvert$}. Thus, a small $K$ value strengthens the trend of encoding low-frequency components of the input sample.

\subsection{Effects of the distribution of the training data}\label{app_discussion_natural}
If the cascaded convolutional decoder network is trained on natural images, then the decoder network is less likely to learn high-frequency components. Previous studies \cite{ruderman1994statistics} have empirically found that natural images were dominated by low-frequency components. Specifically, frequency spectrums of natural images follow a Power-law distribution.
\emph{I.e.}, low-frequency components (\emph{e.g.}, the frequency component {\small$[u,v]$} closed to {\small$[0,0],[0,N-1],[M-1,0]$}, and {\small$[M-1,N-1]$}) have much larger length {\small$\lVert \textbf{g}^{(uv)}\rVert_2=\sqrt{\sum_c \vert G_{uv}^{(c)}\vert^2}$} than other frequency components. Besides, according to rules of the forward propagation in Equation~(\ref{eq:forward}) and the change of {\small$T^{(l,uv)}$} in Equation (\ref{eq:back1}), if the frequency component {\small$\textbf{g}^{(uv)}$} of the input image has a large magnitude, then {\small$\textbf{h}^{(uv)}$} of the output image also has a large magnitude. This means that using natural images as the input strengthens the trend of encoding low-frequency components.

\section{More experimental results}

\subsection{Verifying that a neural network usually learned low-frequent components first.}\label{app:sec:exp:lowfirst}
In section, we provide more experimental results to verify that a neural network usually learned low-frequent components first, which had already been shown in Figure~\ref{fig:bottleneck}(a) in the main paper. Here, we also constructed a cascaded convolutional auto-encoder by using the VGG-16 as the encoder network. The decoder network contained three upconvolutional layers for the CIFAR-10 dataset, and contained three upconvolutional layers for the Broden dataset. Each convolutional/upconvolutional layer in the auto-encoder applied zero-paddings and was followed by a batch normalization layer and an ReLU layer. The auto-encoder was trained using the mean squared error (MSE) loss for image reconstruction. Results in Figure~\ref{app:fig:bottleneck1} verified that the auto-encoder usually learned low-frequent components first and gradually learned higher frequecies. We also attached the generated image below its spectrum map in Figure~\ref{app:fig:bottleneck1_time}, in order to help people understand the learning process of the auto-encoder.

\begin{figure}[tbp]
	\vskip -0.1in
	\centering
	\includegraphics[width=0.98\linewidth]{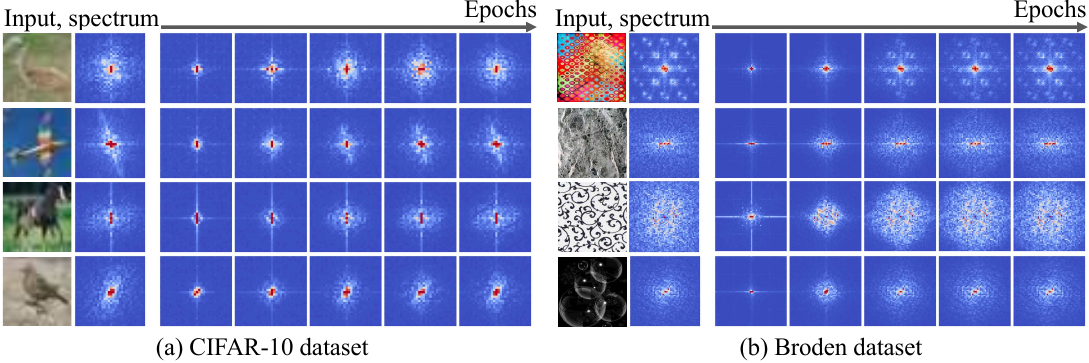}
	\caption{Magnitude maps of feature spectrums of 
		different epochs' network output. Each magnitude map was averaged over all channels. For clarity, we moved low frequencies to the center of the spectrum map, and moved high frequencies to corners of the spectrum map. Note that we set the magnitude of the fundamental frequency to be the same with the magnitude of the second significant frequency. For resutls in (b), we only visualized components in the center of the spectrum map with the range of relatively low frequencies {\small$u\in\{u|0 \le u < M/8\} \cup \{u|7M/8 \le u < M \}; v\in \{v|0 \le v < N/8  \} \cup \{ v|7N/8 \le v < N\}$} for clarity.}
	\label{app:fig:bottleneck1}
	\vskip -0.1in
\end{figure}

\begin{figure}[tbp]
	\vskip -0.1in
	\centering
	\includegraphics[width=\linewidth]{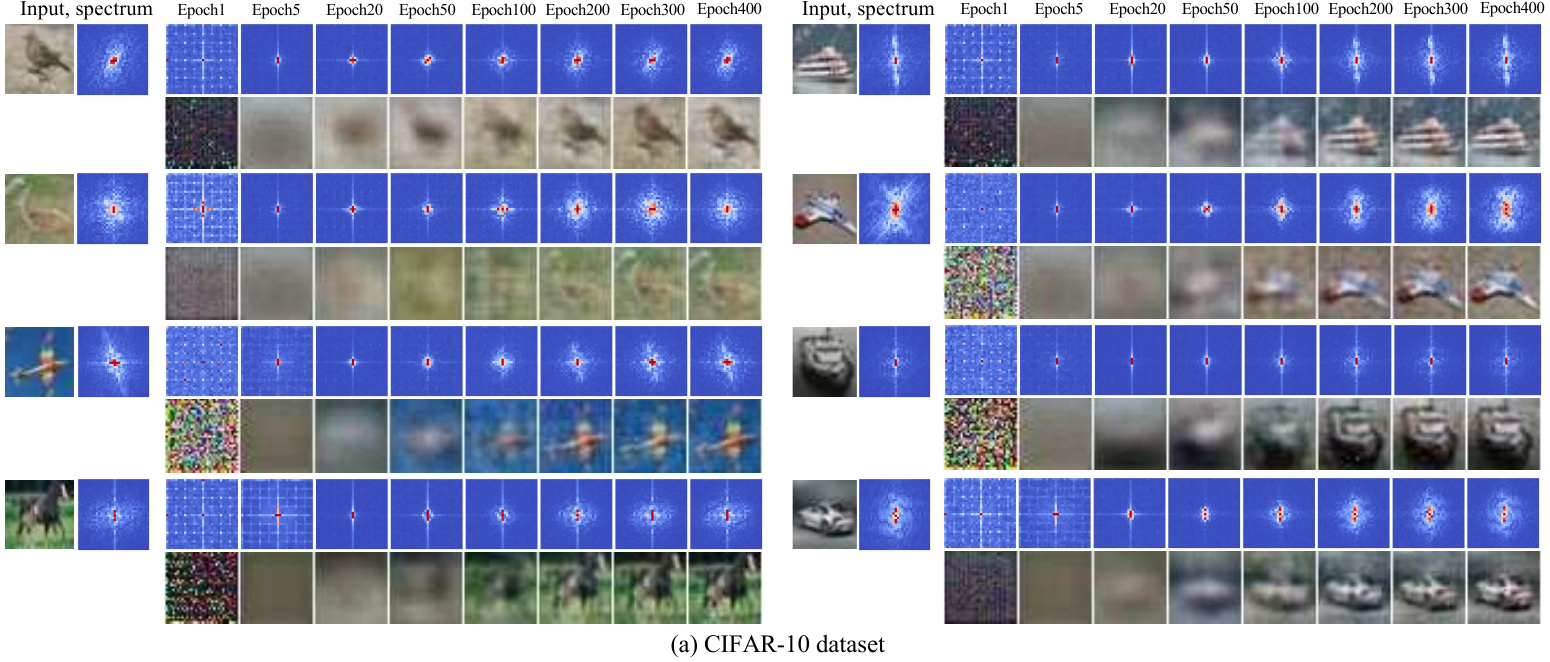}
	\includegraphics[width=\linewidth]{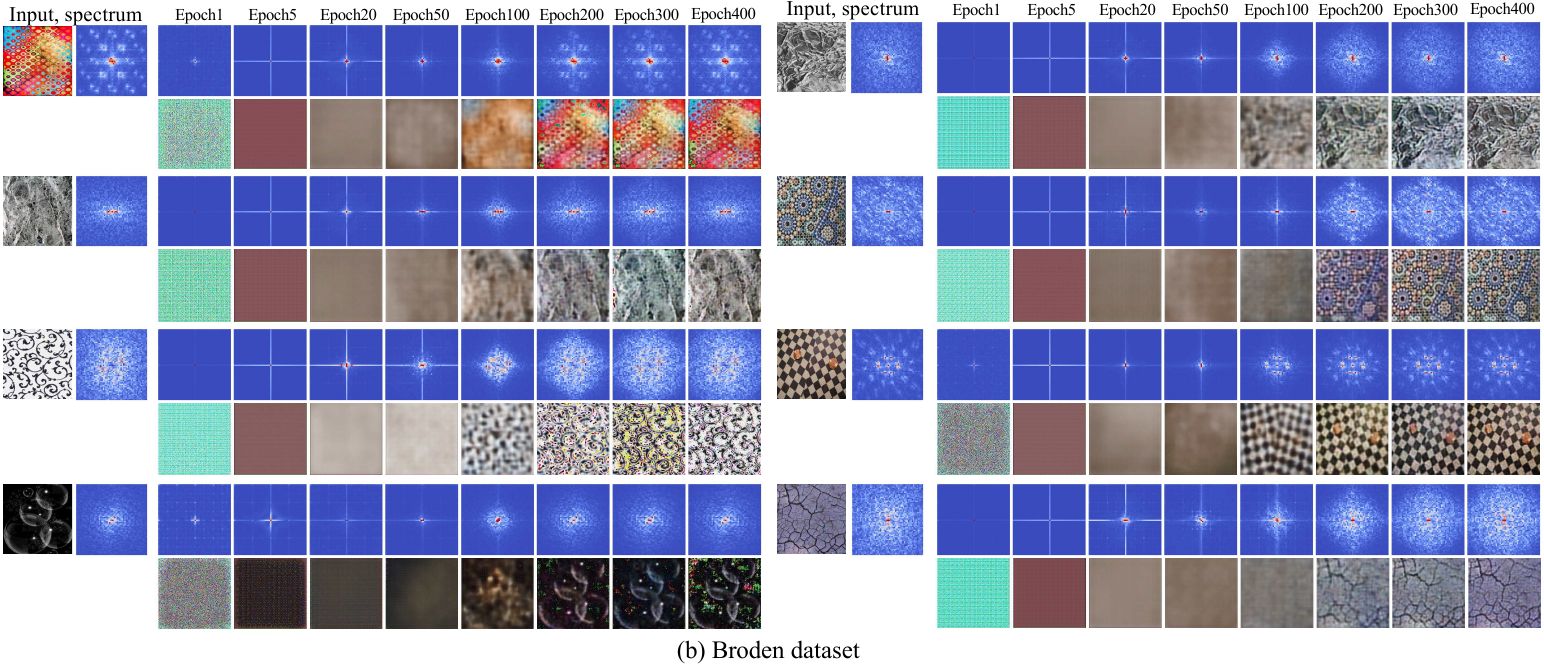}
	\caption{Magnitude maps of feature spectrums and the corresponding generated images of different epochs. Results show that in the very few epochs of the training, the network removed noisy signal caused by the upsampling, to some extent, which were in the grid pattern in the spectrum. After that, the network learned low-frequency components first, and then gradually learned higher frequencies. Each magnitude map in this figure was averaged over all channels. For clarity, we moved low frequencies to the center of the spectrum map, and moved high frequencies to corners of the spectrum map. Note that we set the magnitude of the fundamental frequency to be the same with the frequency that had the second large magnitude.}
	\label{app:fig:bottleneck1_time}
	\vskip -0.1in
\end{figure}

\begin{figure}[tbp]
	\vskip -0.1in
	\centering
	\includegraphics[width=0.98\linewidth]{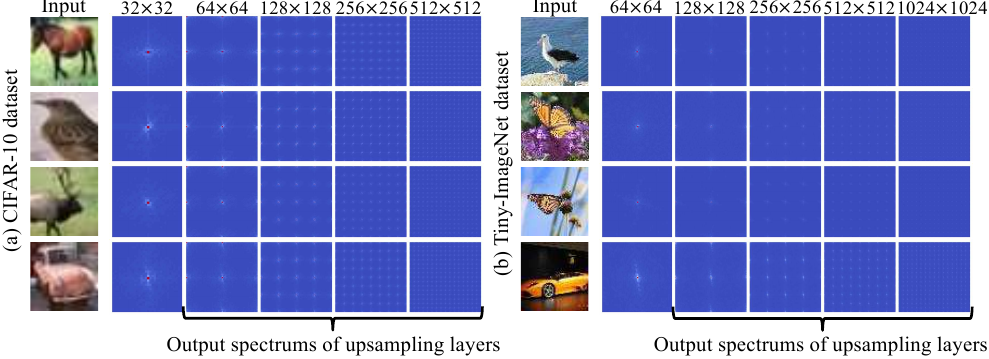}
	\caption{Magnitude maps of feature spectrums after one/two/there/four upsampling layers. Each magnitude map was averaged over all channels. For clarity, we moved low frequencies to the center of the spectrum map, and moved high frequencies to corners of the spectrum map.}
	\label{app:fig:upsampling}
	\vskip -0.1in
\end{figure}

\subsection{Verifying that the upsampling operation made a decoder network repeat strong signals at certain frequencies of the generated image.}\label{app:sec:exp:upsample}
In section, we provide more experimental results to verify that the upsampling operation in the decoder repeats strong frequency components of the input to generate spectrums of upper layers.

First, we conducted experiments to verify Theorem~\ref{th_upconv} in the main paper, which claims that the upsampling operation repeats the strong magnitude of the fundamental frequency {\small$G_{00}^{(c)}$} of the lower layer to different frequency components {\small$\forall c,H_{u^*v^*}^{(c)}$} of the higher layer, where {\small$u^*=0, M_0,2 M_0, 3M_0,\ldots;v^* =0, N_0,2 N_0,3 N_0,\ldots$}. To verify this, given an image, let the image pass through four cascaded upsampling layers. We visualized the feature spectrum generated by each upsampling layer, in order to verify whether the upsampling operation repeated the strong magnitude of the fundamental frequency of the input image to different frequency components of the feature spectrum generated by upsampling layers. Results on the CIFAR-10 dataset and the Tiny-ImageNet dataset in Figure~\ref{app:fig:upsampling} verified Theorem~\ref{th_upconv}.

Second, we provide more results on real neural networks, which have already been shown in Figure~\ref{fig:bottleneck}(b) in the main paper. We also constructed a cascaded convolutional auto-encoder by using the VGG-16 as the encoder network. The decoder network contained four upconvolutional layers. Each convolutional/upconvolutional layer in the auto-encoder applied zero-paddings and was followed by a batch normalization layer and an ReLU layer. The auto-encoder was trained on the Broden dataset using the mean squared error (MSE) loss for image reconstruction. Results in Figure~\ref{app:fig:upsampling2} verified Theorem~\ref{th_upconv}.

\begin{figure}[tbp]
	\vskip -0.1in
	\centering
	\includegraphics[width=0.7\linewidth]{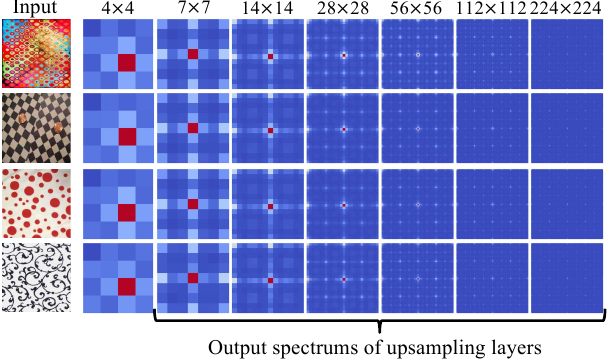}
	\caption{Magnitude maps of feature spectrums after one/two/there/four/five/six upsampling layers. Each magnitude map was averaged over all channels. For clarity, we moved low frequencies to the center of the spectrum map, and moved high frequencies to corners of the spectrum map.}
	\label{app:fig:upsampling2}
	\vskip -0.1in
\end{figure}
\begin{figure}[tbp]
	\vskip -0.1in
	\centering
	\includegraphics[width=0.95\linewidth]{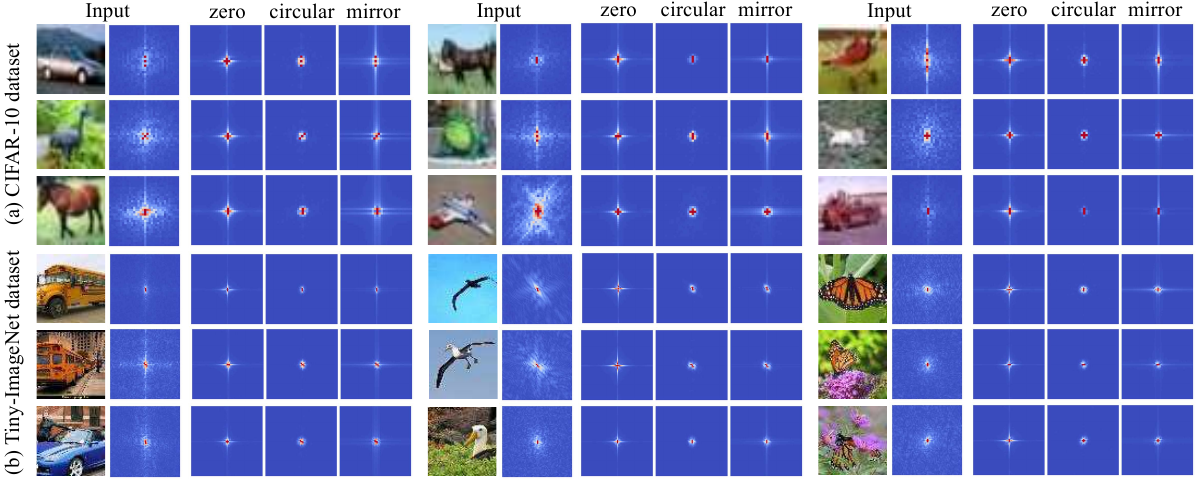}
	\caption{A network with zero-padding operations usually strengthened more low-frequency components than a network with circular padding operations and a network with mirror padding operations. Here, each magnitude map of the feature spectrum was averaged over all channels. For clarity, we move low frequencies to the center of the spectrum map, move high frequencies to corners of the spectrum map, and set the magnitude of the fundamental frequency to be the same with the frequency that has the second large magnitude.}
	\label{app:fig:zero}
	\vskip -0.1in
\end{figure}

\subsection{Verifying that the zero-padding operation strengthened the encoding of low-frequency components.}\label{app:sec:exp:zero}

In section, we provide more experimental results to verify that the zero-padding operation strengthened the encoding of low-frequency components, which had already been shown in Figure~\ref{fig:remark2_3}(c) in the main paper. Here, we also constructed the following three baseline networks. The first baseline network contained 5 convolutional layers, and each layer applied zero-paddings. Each convolutional layer contained 16 convolutional kernels (kernel size was 7$\times $7), except for the last layer containing 3 convolutional kernels. The second baseline network and the third baseline network were constructed by replacing all zero-padding operations with circular padding operations and replacing all zero-padding operations with mirror padding operations, respectively. Results in Figure~\ref{app:fig:zero} verified that the zero-padding operation strengthened the encoding of low-frequency components.

\subsection{Verifying that a deep network strengthened low-frequency components.}\label{app:sec:exp:factors:L}

In section, we provide more experimental results to verify that a deep network strengthened low-frequency components, which had already been shown in Figure~\ref{fig:remark2_3}(a) in the main paper. Here, we also constructed a network with 50 convolutional layers. Each convolutional layer applied zero-paddings to avoid changing the size of feature maps, and was followed by an ReLU layer. We visualized feature spectrums of different convolutional layers. Results on the CIFAR-10 dataset and the Tiny-ImageNet dataset in Figure~\ref{app:fig:depth} show that magnitudes of low-frequency components increased along with the network layer number.

\begin{figure}
	\vskip -0.1in
	\centering
	\includegraphics[width=0.7\linewidth]{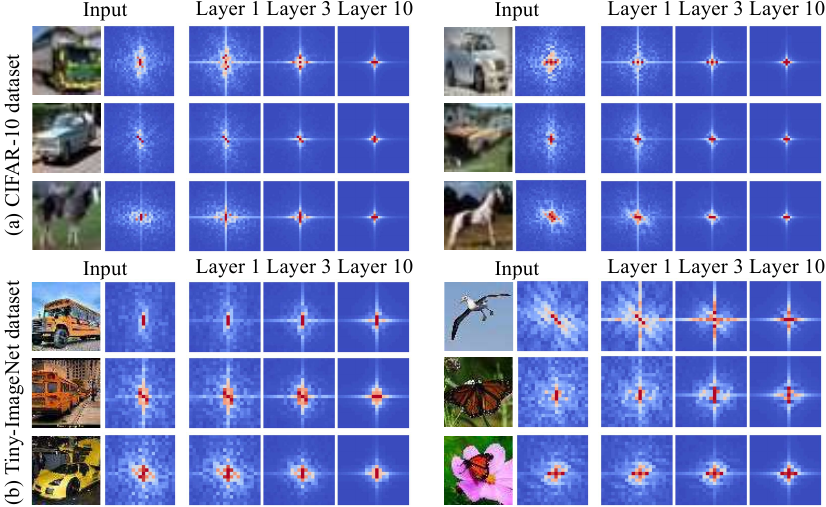}
	\caption{Comparing feature spectrums of different layers. Results show that higher layers of a network usually generated features with more low-frequency components. For clarity, we move low frequencies to the center of the spectrum map, move high frequencies to corners of the spectrum map, and set the magnitude of the fundamental frequency to be the same with the frequency that has the second large magnitude. For resutls in (b), we only visualized components in the center of the spectrum map with the range of relatively low frequencies {\small$u\in\{u|0 \le u < M/6\} \cup \{u|5M/6 \le u < M \}; v\in \{v|0 \le v < N/6  \} \cup \{ v|5N/6 \le v < N\}$} for clarity.}
	\label{app:fig:depth}
	\vskip -0.1in
\end{figure}

\subsection{Verifying that a larger absolute mean value {\small$\mu_l$} of each $l$-th layer's parameters strengthened low-frequency components.}\label{app:sec:exp:factors:mu}

In section, we provide more experimental results to verify that a larger absolute mean value {\small$\mu_l$} of each $l$-th layer's parameters strengthened low-frequency components, which had already been shown in Figure~\ref{fig:remark2_3}(b) in the main paper. Here, we also applied a network architecture with 5 convolutional layers. Each layer contained 16 convolutional kernels (kernel size was 9$\times $9), except for the last layer containing 3 convolutional kernels. Based on this architecture, we constructed three networks, whose parameters were sampled from Gaussian distributions {\small$\mathcal{N}(\mu=0,\sigma^2=0.01^2)$}, {\small$\mathcal{N}(\mu=0.001,\sigma^2=0.01^2)$}, and {\small$\mathcal{N}(\mu=0.01,\sigma^2=0.01^2)$}, respectively. Results on the CIFAR-10 dataset and the Tiny-ImageNet dataset in Figure~\ref{app:fig:mu} show that magnitudes of low-frequency components increased along with the absolute mean value of parameters.

\begin{figure}
	\vskip -0.1in
	\centering
	\includegraphics[width=0.7\linewidth]{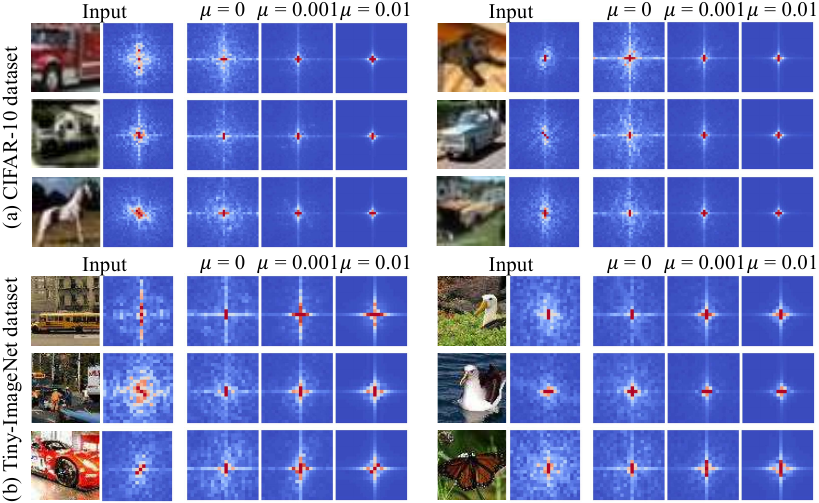}
	\caption{A network whose convolutional weights had a mean value significantly biased from 0 usually strengthened low-frequency components, but weakened high-frequency components. Here, each magnitude map of the feature spectrum was averaged over all channels. For clarity, we moved low frequencies to the center of the spectrum map, moved high frequencies to corners of the spectrum map, and set the magnitude of the fundamental frequency to be the same with the frequency that has the second large magnitude. For resutls in (b), we only visualized components in the center of the spectrum map with the range of relatively low frequencies {\small$u\in\{u|0 \le u < M/6\} \cup \{u|5M/6 \le u < M \}; v\in \{v|0 \le v < N/6  \} \cup \{ v|5N/6 \le v < N\}$} for clarity.}
	\label{app:fig:mu}
	\vskip -0.1in
\end{figure}

\subsection{Discussions on the curse of dimension}\label{app:sec:dimension}

In this section, we discuss the problem of the curse of dimension, when we compute the cosine similarity of two high-dimensional vectors with as many as {\small$32^2$}, {\small$64^2$} or {\small$224^2$} dimensions. In general, for each pair of extremely high-dimensional vectors, it's hard for these vectors to have very high cosine similarity. It is because even if the noisy differences between many pairs of dimensions can be ignored, in the process of calculating the sum of squares, these noisy differences will accumulate. Therefore, the cosine similarity of two extremely high-dimensional vectors will not be particularly large.

\subsection{Verifying that Assumption 3.1 can be applied to fully trained DNNs}\label{app:sec:ass}

Assumption~\ref{ass:T} shows that in early training of a DNN, all elements in $T^{(l,uv)}$ are irrelevant to each other, and $\forall l\ne l'$, elements in $T^{(l,uv)}$ and $T^{(l',uv)}$ are irrelevant to each other.
In this section, we further conducted experiments to verify that such irrelevant relationships also existed in a fully trained DNN. To this end, we constructed a cascaded convolutional auto-encoder by using the VGG-16 as the encoder network. The decoder network contained ten convolutional layers, where the 1st, 3rd, 5th, 7th, and 9th layers were traditional convolutional layers, and the 2nd, 4th, 6th, 8th, and 10th layers were upconvolutional layers. Each convolutional/upconvolutional layer applied zero-paddings and was followed by a batch normalization layer and an ReLU layer. The network was trained on the Tiny-ImageNet dataset using the mean squared error (MSE) loss for image reconstruction.

Let the above auto-encoder be trained to convergence. Then, we computed the Pearson's correlation coefficient between each random pair of variables $\lvert \Delta T^{(l,uv)}_{d_1c_1} \rvert$ and $\lvert \Delta T^{(l,uv)}_{d_2c_2}\rvert$ through different images, denoted by $r(\lvert \Delta T^{(l,uv)}_{d_1c_1} \rvert,\lvert \Delta T^{(l,uv)}_{d_2c_2} \rvert)$, to measure the relevance between two elements $T^{(l,uv)}_{d_1c_1}$ and $T^{(l,uv)}_{d_2c_2}$ in $T^{(l,uv)}$. Here, $\Delta T^{(l,uv)}$ denoted the change of $T^{(l,uv)}$ when we updated parameters $\textbf{W}$ for a single gradient-descent step on a single input sample. Results in Table~\ref{app:tab:ass} show that even when the network was fully trained, different elements in $T^{(l,uv)}$ had low Pearson's correlation coefficient $r$. This proved that our assumption that all elements in $T^{(l,uv)}$ were irrelevant to each other was reasonable. The last three convolutional layers in the decoder showed a larger Pearson's correlation coefficient, because network parameters close to the output layer had been converged to the principle feature direction of each category. Nevertheless, our experiments showed that for most layers, we could keep Assumption~\ref{ass:T}, which enabled to us to prove that convolution operations in these layers weakened high-frequency components. 
\begin{table}[tbp]
	\caption{Pearson's correlation coefficient between each random pair of variables $\lvert \Delta T^{(l,uv)}_{d_1c_1} \rvert$ and $\lvert \Delta T^{(l,uv)}_{d_2c_2}\rvert$ through different images.}
	\label{app:tab:ass}
	\centering
	\resizebox{0.7\linewidth}{!}{
		\begin{tabular}{c|c}
			\toprule
			Depth of the decoder network&  $\mathbb{E}_{u,v,d_1,c_1,d_2,c_2}[r(\lvert \Delta T^{(l,uv)}\_{d_1c_1} \rvert,\lvert \Delta T^{(l,uv)}\_{d_2c_2} \rvert)]$\\
			\midrule
			1 & 0.011  \\
			2 & -0.002 \\
			3 & 0.018 \\
			4 & -0.001 \\
			5 & -0.013 \\
			6 & 0.026 \\
			7 & 0.018 \\
			\bottomrule
	\end{tabular}}
\end{table}

\subsection{Effects of large absolute mean values on layers with different depth}\label{app:sec:large_mu_depth}

We conducted experiments to measure the effects of large absolute mean values on layers with different depth. To this end, we compared spectrums of output features, when we set large absolute mean values for parameters in different convolutional layers. Therefore, we constructed five convolutional networks with the same architecture for comparison. Each convolutional network had five convolutional layers. To construct the $l$-th network for comparison, we sampled parameters of the $l$-th convolutional layer from the Gaussian distribution $\mathcal{N}(\mu=0.1,\sigma^2=0.1^2)$, and sampled parameters of the remaining four convolutional layers from the Gaussian distribution $\mathcal{N}(\mu=0,\sigma^2=0.1^2)$. Besides, based on this architecture, we also constructed the sixth network by let parameters of all layers be sampled from the Gaussian distribution $\mathcal{N}(\mu=0,\sigma^2=0.1^2)$.

Figure~\ref{app:fig:large_mu_depth} shows results on the Broden dataset. Compared with the sixth network that all parameters had zero mean (see Figure~\ref{app:fig:large_mu_depth}(a)), all other networks (see Figure~\ref{app:fig:large_mu_depth}(b1-b5)), whose parameters in a certain layer had a large absolute mean value, weakened high-frequency components. Besides, no matter which layer had parameters of a large absolute mean value, there was no significant difference between the five networks in weakening the encoding of high-frequency components.

\begin{figure}
	\centering
	\includegraphics[width=0.7\linewidth]{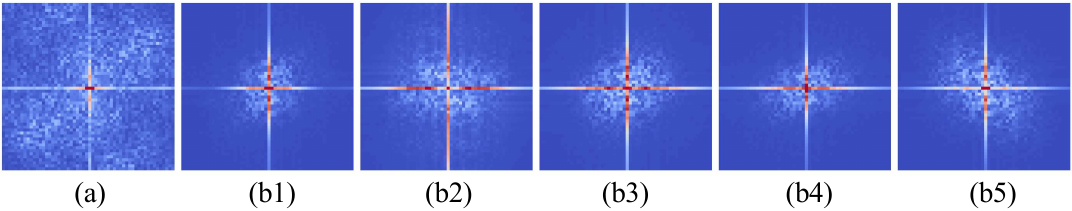}
	\caption{(a) The spectrum generated by the sixth decoder, whose parameters in all layers had zero mean.
(b1-b5) Spectrums generated by ﻿the first five decoders, whose parameters in a certain layer (the 1st, 2nd, 3rd, 4th, and 5th, respectively) had a large absolute mean value. Results show that no matter which layer of the network had a larger absolute mean value of parameters, there was no significant difference in weakening the encoding of high-frequency components.}
	\label{app:fig:large_mu_depth}
	\vskip -0.1in
\end{figure}

\subsection{Details about the frequency shift}\label{app:sec:shift}

In this section, we provide examples to introduce details about how to shift each salient frequency component {\small$[u,v]$} in the input {\small$x$} to {\small$[u+\Delta u,v]$} or {\small$[u-\Delta u,v]$} towards higher frequencies, as Figure~\ref{app:fig:shift} shows. Note that we move low frequencies to the center of the spectrum map, and move high frequencies to the corners of the spectrum map for clarity.

\begin{figure}
	\centering
	\includegraphics[width=0.7\linewidth]{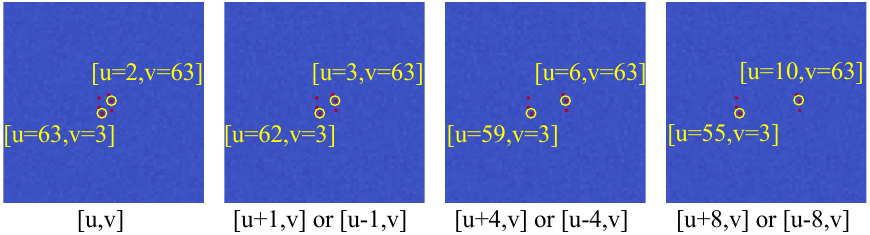}
	\caption{Examples of shifting each salient frequency component {\small$[u,v]$} in the input to {\small$[u+\Delta u,v]$} or {\small$[u-\Delta u,v]$} towards higher frequencies, where {\small$\Delta u = 1,4,8$}. For clarity, we move low frequencies to the center of the spectrum map, and move high frequencies to the corners of the spectrum map.}
	\label{app:fig:shift}
	\vskip -0.1in
\end{figure}

\subsection{More experimental details}\label{app:sec:details}
Table~\ref{app:tab:details} reports the number of epochs for the training of each model and its fitting error $\mathbb{E}_x[\frac{\lVert x-\hat{x} \rVert_2^2}{N}]$, where $N$ denoted the number of pixels in the image.

\begin{table}[htbp]
	\caption{Number of epochs for the training of each model and its fitting error.}
	\label{app:tab:details}
	\centering
	\resizebox{\linewidth}{!}{
		\begin{tabular}{l|cc}
			\toprule
			& \# training epoch & fitting error \\
			\midrule
			The model used in verifying that a neural network usually learned low frequent components first &400  & 9.35e-3  \\
			The model used in verifying the repeat of certain frequencies. & 10 & 6.97e-2 \\
			\bottomrule
	\end{tabular}}
\end{table}

\end{document}